\documentclass[10pt,journal,compsoc]{IEEEtran}

\usepackage{amsmath,amsfonts}
\usepackage{algorithmic}
\usepackage{array}
\usepackage[caption=false,font=footnotesize,labelfont=rm,textfont=rm]{subfig}
\usepackage{textcomp}
\usepackage{stfloats}
\usepackage{url}
\usepackage{verbatim}
\usepackage{graphicx}
\usepackage{cite}
\usepackage{amssymb}
\usepackage{amsthm}
\usepackage{flushend}
\usepackage{dsserif}

\usepackage[ruled]{algorithm2e}

\usepackage{subfloat}

\DeclareMathOperator*{\argmax}{arg\,max}

\def\BibTeX{{\rm B\kern-.05em{\sc i\kern-.025em b}\kern-.08em
    T\kern-.1667em\lower.7ex\hbox{E}\kern-.125emX}}
\usepackage{balance}

\newtheorem{theorem}{Theorem}
\newtheorem{proposition}{Proposition}
\newtheorem{definition}{Definition}

\newtheorem{remark}{Remark}

\begin{document}

\title{Safe Reinforcement Learning\\with Dual Robustness}

\author{Zeyang Li, Chuxiong Hu, Yunan Wang, Yujie Yang, Shengbo Eben Li
    \IEEEcompsocitemizethanks{
        \IEEEcompsocthanksitem Zeyang Li, Chuxiong Hu and Yunan Wang are with the Department of Mechanical Engineering, Tsinghua University, Beijing 100084, China (email: li-zy21@mails.tsinghua.edu.cn; cxhu@tsinghua.edu.cn; wang-yn22@mails.tsinghua.eud.cn).
        \IEEEcompsocthanksitem Yujie Yang and Shengbo Eben Li are with the School of Vehicle and Mobility, Tsinghua University, Beijing 100084, China (email: yangyj21@mails.tsinghua.edu.cn; lishbo@tsinghua.edu.cn).
    }
    \thanks{(Corresponding author: Chuxiong Hu.)}
    }

\IEEEtitleabstractindextext{
    \begin{abstract}
        Reinforcement learning (RL) agents are vulnerable to adversarial disturbances, which can deteriorate task performance or break down safety specifications. Existing methods either address safety requirements under the assumption of no adversary (e.g., safe RL) or only focus on robustness against performance adversaries (e.g., robust RL). Learning one policy that is both safe and robust under any adversaries remains a challenging open problem. The difficulty is how to tackle two intertwined aspects in the worst cases: feasibility and optimality. The optimality is only valid inside a feasible region (i.e., robust invariant set), while the identification of maximal feasible region must rely on how to learn the optimal policy. To address this issue, we propose a systematic framework to unify safe RL and robust RL, including the problem formulation, iteration scheme, convergence analysis and practical algorithm design. The unification is built upon constrained two-player zero-sum Markov games, in which the objective for protagonist is twofold. For states inside the maximal robust invariant set, the goal is to pursue rewards under the condition of guaranteed safety; for states outside the maximal robust invariant set, the goal is to reduce the extent of constraint violation. A dual policy iteration scheme is proposed, which simultaneously optimizes a task policy and a safety policy. We prove that the iteration scheme converges to the optimal task policy which maximizes the twofold objective in the worst cases, and the optimal safety policy which stays as far away from the safety boundary. The convergence of safety policy is established by exploiting the monotone contraction property of safety self-consistency operators, and that of task policy depends on the transformation of safety constraints into state-dependent action spaces. By adding two adversarial networks (one is for safety guarantee and the other is for task performance), we propose a practical deep RL algorithm for constrained zero-sum Markov games, called dually robust actor-critic (DRAC). The evaluations with safety-critical benchmarks demonstrate that DRAC achieves high performance and persistent safety under all scenarios (no adversary, safety adversary, performance adversary), outperforming all baselines by a large margin.
    \end{abstract}
    
}


\maketitle


\section{Introduction}

\IEEEPARstart{R}{einforcement} learning (RL) has demonstrated tremendous performance in various fields, including games \cite{silver2017mastering}, robotics \cite{hwangbo2019learning, huang2022reward}, and autonomous driving \cite{feng2023dense, guan2021integrated}. However, it remains challenging to deploy RL methods on real-world complex control tasks. The challenges are twofold. First, most real-world tasks not only require the RL agents to maximize the total rewards but also demand strict safety constraint satisfaction, since violating these constraints can lead to severe consequences \cite{garcia15a}. Second, there are nonnegligible gaps between simulation platforms and real-world scenarios, such as model mismatches, sensory noises and environmental perturbations, which requires strong robustness of RL agents \cite{ju2022transfer, li2023relaxed}.

Safe RL is a research area that aims at learning policies that satisfy safety constraints \cite{garcia15a, shengbo2018reinforcement}. There are mainly two categories of safety formulations in existing safe RL algorithms: trajectory safety and stepwise safety. In the trajectory safety formulation of safe RL, the objective is to find a policy that maximizes total rewards under the condition that the expectation of trajectory costs is below a predefined threshold. Many works adopt the method of Lagrange multipliers to solve the constrained optimization problem. Ha et al. combine soft actor-critic algorithm with Lagrange multipliers and perform dual ascent on the Lagrangian function \cite{ha2020learning}. Chow et al. impose constraints for conditional value-at-risk of the cumulative cost and derive the gradient of the Lagrangian function \cite{chow2017risk}. Tessler et al. augment the reward function with additional penalty signals, which guide the policy toward a constraint-satisfying solution \cite{tessler2018reward}.
Trust region method is also utilized for constrained optimization, in which the objective and the constraint are both approximated with low-order functions, yielding a local analytic solution for policy improvement. Achiam et al. propose the constrained policy optimization algorithm, in which both objective and cost value constraint are approximated with linear functions \cite{achiam2017constrained}. Yang et al. maximize rewards using trust region optimization and project the policy back onto the constraint set \cite{yang2020projection}.
In the stepwise safety formulation of safe RL, the objective is to find a policy that maximizes total rewards while ensuring strict constraint satisfaction at every state the agent visits. This line of work utilizes the rigorous definition and theoretical guarantee of safety in the safe control community \cite{korda2014convex, ames2016control}. The crucial insight is that persistence safety is only possible for a subset of the constraint set, called invariant set \cite{korda2014convex}. Some works achieve set invariance with energy functions, such as control barrier function \cite{ames2016control} and safety index \cite{wei2019safe}. Ma et al. jointly optimize the control policy and safety index \cite{ma2022joint}. Yang et al. jointly optimize the control policy and control barrier function \cite{yang2023model}. Besides energy functions, Hamilton-Jacobi reachability analysis \cite{margellos2011hamilton} is also utilized for synthesis of invariant sets. Fisac et al. develop an RL approach for computing the safety value functions (representation of invariant sets) in Hamilton-Jacobi reachability analysis \cite{fisac2019bridging}. Yu et al. further utilize the learned safety value function to perform constrained policy optimization \cite{yu2022reachability}. Despite the fruitful research on safe RL, previous methods rarely account for external disturbances, which may severely harm the safety-preserving capability of safe RL algorithms.

Robust RL is a research area that aims at learning policies that enjoy performance robustness against uncertainties \cite{moos2022robust}. Existing works consider uncertainties of different elements in RL, such as states \cite{zhang2020state}, actions \cite{tessler2019action}, transition probabilities \cite{nilim2003transition} and rewards \cite{liu2020reward}. Based on the viewpoint of robust control, model mismatches, sensory noises and environmental perturbations can all be viewed as external disturbances in the system, so in this paper we focus on robustness against disturbances under the setting of two-player zero-sum Markov games, in which control inputs serve as the protagonist and disturbances serve as the adversary. Two-player zero-sum Markov game is a generalization of the standard Markov Decision Process (MDP), in which the protagonist tries to maximize the accumulated reward while the adversary tries to minimize it \cite{perolat2015approximate}. Pinto et al. propose the idea of robust adversarial RL, in which the protagonist policy and the adversary policy are parameterized by neural networks and jointly trained \cite{pinto2017robust}. Their method improves training stability and the learned agents enjoy strong robustness in performance. Zhu et al. propose the minimax Q-network for two-player zero-sum Markov games \cite{zhu2020online}. Tessler et al. propose probabilistic action robust MDP and noisy action robust MDP, which are related to common forms of uncertainty in robotics \cite{tessler2019action}. They analyze the two forms of MDP in the tabular setting and design corresponding deep RL algorithms.

However, external disturbances not only can deteriorate the task performance (i.e., making agents unable to accomplish their goals), but also may break down safety specifications (i.e., leading to catastrophic damage on the entire system).
Existing safe RL methods seldom consider external perturbations. Their agents are vulnerable to performance attacks, and the safety-preserving ability no longer holds under safety attacks. Existing robust RL methods only consider robustness against performance attacks and lack safety-preserving ability even without adversaries. To the best of our knowledge, there are no RL algorithm that can learn one policy that is robust to both safety and performance attacks. The difficulty is how to simultaneously handle two intertwined aspects, i.e., feasibility and optimality, in the worst cases. The former refers to the fact that there exists no feasible solution for a policy that can keep the system safe under worst-case safety attacks in a certain region inside the constraint set. The latter refers to attaining the highest rewards under worst-case performance attacks. The two aspects are heavily intertwined, since optimality is only valid inside a feasible region (i.e., robust invariant set), while the identification of maximal feasible region must rely on how to learn the optimal policy.

To overcome the aforementioned challenges, this paper builds a theoretical framework to unify safe RL and robust RL. A constrained zero-sum Markov game is essentially a constrained optimization problem, the key of which is designing a twofold objective. For states inside the maximal robust invariant set, the objective is to maximize the value function while satisfying the constraints of admissible control inputs, which are specified by the safety value function. For states outside the maximal robust invariant set, the objective for these states is to reduce the degree of constraint violation, since the system will violate the safety constraints inevitably and it is meaningless to pursue rewards. The crux is to propose an iteration scheme for this constrained optimization problem, prove its convergence, and design a deep RL algorithm for practical implementation. The contributions of this paper are summarized as follows.
\begin{itemize}
    \item We propose a dual policy iteration scheme to solve constrained zero-sum Markov games, which jointly iterates two policies: task policy for maximization of the twofold objective and safety policy for identification of the robust invariant set. We establish the self-consistency conditions of safety value functions, which are utilized to iterate the safety policy, i.e., alternating between safety policy evaluation and safety policy improvement. The task policy is optimized by alternating between task policy evaluation and task policy improvement. The latter consists of two parts, which are consistent with the twofold objective of constrained zero-sum Markov games. For states inside the current invariant set specified by the current safety policy, greedy search is performed under the constraint of safety value function. For states outside the current invariant set, the task policy copies the safety policy, for the purpose of reducing the extent of future constraint violation.
    \item The convergence of this iteration scheme is proved by separately discussing different behavioral properties of safety policy and task policy. For safety policy, its convergence is established by exploiting the monotone contraction property of the safety self-consistency operators. For task policy, we can establish a well-defined unconstrained zero-sum Markov game on the robust invariant set of a specified safety policy, transforming the original safety constraints into the form of state-dependent action spaces. We prove that the proposed dual policy iteration scheme converges to the optimal task policy (i.e., the policy attaining the maximal objective for the proposed constrained optimization problem) and the optimal safety policy (i.e., the policy seeking the highest safety values).
    \item We propose a practical deep RL algorithm for constrained zero-sum Markov games, called dually robust actor-critic (DRAC), which can learn one policy that is robust to both performance and safety attacks. Since it is intractable to conduct thorough task policy evaluation and safety policy evaluation on high-dimensional continuous spaces, a safety adversary network and a performance adversary network are additionally introduced for practical implementation. Since there are infinite constraints on the task policy, we also introduce a Lagrange multiplier network to facilitate the constrained optimization process.
\end{itemize}

\section{Background}

\subsection{Safe RL}

In this work, we consider the stepwise deterministic safety specification for safe RL, which aims to ensure that the learned control policy satisfies the safety constraints on every state it visits \cite{yu2022reachability, yang2023model}. Consider an MDP with deterministic system dynamics $\left(\mathcal{X}, \mathcal{U}, f, r, h, \gamma\right)$, in which $\mathcal{X}$ denotes the state space, $\mathcal{U}$ denotes the control space, $f: \mathcal{X}\times\mathcal{U} \rightarrow \mathcal{X}$ denotes the system dynamics, $\pi: \mathcal{X} \rightarrow \Delta(\mathcal{U})$ denotes the control policy ($\Delta(\mathcal{U})$ represents the set of probability distributions on $\mathcal{U}$), $\gamma$ denotes the discount factor, $r: \mathcal{X} \times \mathcal{U} \rightarrow \mathbb{R}$ denotes the reward function, and $h: \mathcal{X} \rightarrow \mathbb{R}$ denotes the constraint function. Let $d$ denote the initial state distribution. The problem formulation is specified as follows.
\begin{equation}
    \nonumber
    \begin{gathered}
        \max _\pi \mathop{\mathbb{E}}\limits_{x_0\sim d}\left\{\sum_{t=0}^{\infty} \gamma^t r\left(x_t, u_t\right)\right\} \\
        \begin{aligned}
            \text { s.t. \quad }&x_{t+1}=f\left(x_t, u_t\right), u_t\sim\pi(\cdot\mid x_t), \\
            &h\left(x_t\right) \geq 0, t=0,1,2, \ldots, \infty.
        \end{aligned}
    \end{gathered}
\end{equation}


\subsection{Robust RL}
The foundation of robust RL lies in the theory of two-player zero-sum Markov games, in which the protagonist aims to maximize the accumulated reward while the adversary tries to decrease it \cite{perolat2015approximate}. A two-player zero-sum Markov game can be represented by a tuple $\left(\mathcal{X}, \mathcal{U}, \mathcal{A}, p, r, \gamma\right)$, in which $\mathcal{X}$ denotes the state space, $\mathcal{U}$ denotes the protagonist action space, $\mathcal{A}$ denotes the adversary action space, $p: \mathcal{X} \times \mathcal{U} \times \mathcal{A} \rightarrow \Delta(\mathcal{X})$ denotes the transition probability, $\pi: \mathcal{X} \rightarrow \Delta(\mathcal{U})$ denotes the protagonist policy, $\mu: \mathcal{X} \rightarrow \Delta(\mathcal{A})$ denotes the adversary policy, $\gamma$ denotes the discount factor, and $r: \mathcal{X} \times \mathcal{U} \times \mathcal{A} \rightarrow \mathbb{R}$ denotes the reward function. Let $d$ denote the initial state distribution. The problem formulation is specified as follows.
\begin{equation}
    \nonumber
    \begin{gathered}
        \max _\pi \ \min _\mu \mathop{\mathbb{E}}\limits_{x_0\sim d}\left\{\sum_{t=0}^{\infty} \gamma^t r\left(x_t, u_t, a_t\right)\right\} \\
        \begin{aligned}
            \text { s.t. \quad }&x_{t+1}=f\left(x_t, u_t, a_t\right),\\
            &u_t\sim \pi \left( \cdot \mid x_t \right) , a_t\sim \mu \left( \cdot \mid x_t \right).
        \end{aligned}
    \end{gathered}
\end{equation}


\section{Dual Policy Iteration for Constrained Zero-sum Markov Games}

In this section, we develop a theoretical framework for constrained zero-sum Markov games, in which the control inputs serve as the protagonist and the external disturbances serve as the adversary. Adversarial disturbances have two kinds of effects on the system. They can either attack the safety specifications or attack the task performance. Since safety is the top priority, we must first ensure constraint satisfaction under worst-case safety adversaries, which requires identification of the maximal robust invariant set. Then we need to further impose performance robustness on the control policy. We propose a dual policy iteration scheme to accomplish this goal.

In this work, the constrained zero-sum Markov game is denoted by a tuple $\mathcal{M}=\left(\mathcal{X}, \mathcal{U}, \mathcal{A}, f, r, h, \gamma\right)$, in which $\mathcal{X}$ represents the (finite) state space, $\mathcal{U}$ represents the (finite) protagonist action space, $\mathcal{A}$ represents the (finite) adversary action space, $f: \mathcal{X} \times \mathcal{U} \times \mathcal{A} \rightarrow \mathcal{X}$ represents the deterministic system dynamics, $\gamma$ represents the discount factor, $r: \mathcal{X} \times \mathcal{U} \times \mathcal{A} \rightarrow \mathbb{R}$ represents the reward function, and $h: \mathcal{X} \rightarrow \mathbb{R}$ denotes the constraint function. If there are multiple safety requirements, i.e., $\left\{h_1(x)\geq0,h_2(x)\geq0,\cdots\right\}$, the constraint function can be specified as $h(x)=\min\left\{h_1(x),h_2(x),\cdots\right\}$. Let $\pi: \mathcal{X} \rightarrow \Delta(\mathcal{U})$ represent the protagonist policy and $\mu: \mathcal{X} \rightarrow \Delta(\mathcal{A})$ represent the adversary policy.

\subsection{Safety Value Function}

Safety is of the utmost importance. When external disturbances exist, we must ensure safety under worst-case scenarios. A fundamental fact from safe control research is that only those states inside a subset of the constraint set, called robust invariant set, can achieve persistent safety \cite{margellos2011hamilton}. For states outside the robust invariant set, there always exist some adversarial disturbances that will drive the system to violate the safety constraints regardless of what control policy is adopted. Therefore, enforcing safety is essentially equivalent to constraining the system states inside the robust invariant set.

Hamilton-Jacobi reachability analysis \cite{margellos2011hamilton} provides a formal mathematical tool to compute the robust invariant set of arbitrary nonlinear systems with bounded disturbances. However, it suffers from the curse of dimensionality and is intractable for high-dimensional systems. Recently, some pioneering works have migrated Hamilton-Jacobi reachability analysis to model-free RL \cite{fisac2019bridging,yu2022reachability}. However, these works assume that there are no disturbances in the system and only compute the standard invariant set, whose persistent safety can be compromised when disturbances exist. In this work, we make use of the safety value function in Hamilton-Jacobi reachability analysis and develop RL techniques to solve for the robust invariant set.

We begin by defining the following safety values. Note that we focus on action values, which benefit model-free RL. Our results can be easily extended to state values. The notation $\mathop{\min}\limits_{\lambda\sim d}f(\lambda)$ denotes the minimum value of $f(\lambda)$ on the support set of distribution $d$.

\begin{definition}[safety value functions]
    \hfill
    \label{safety value functions}
    \begin{enumerate}
        \item Let $\tau(x,u,a)$ denote an infinite-horizon trajectory starting from $(x,u,a)$, i.e.,
            \begin{equation}
                \nonumber
                \tau(x, u, a) \triangleq\left\{x_0=x, u_0=u, a_0=a, x_1, u_1, a_1, \cdots\right\}.
            \end{equation}
        The safety value of a trajectory $\tau(x,u,a)$ is defined as
            \begin{equation}
                \nonumber
                Q_h^\tau(x, u, a)=\min _{t \in \mathbb{N}}\left\{h\left(x_t\right) \mathrel{}\middle|\mathrel{} x_t \in \tau(x, u, a)\right\}.
            \end{equation}
        \item The safety value function of a protagonist policy $\pi_h$ and an adversary policy $\mu_h$ is defined as
            \begin{equation}
                \label{definition of Qh pi mu}
                Q_h^{\pi_h, \mu_h}(x, u, a)=\min _{\tau \sim \pi_h, \mu_h}\left\{Q_h^\tau(x, u, a)\right\}.
            \end{equation}
        \item The safety value function of a protagonist policy $\pi_h$ is defined as
            \begin{equation}
                \label{definition of Qh pi}
                \begin{aligned}
                    Q_h^{\pi_h}(x, u, a)&=\min _{\mu_h} Q_h^{\pi_h, \mu_h}(x, u, a)\\&=\min _{\mu_h} \min _{\tau \sim \pi_h, \mu_h}\left\{Q_h^\tau(x, u, a)\right\}.
                \end{aligned}
            \end{equation}
        \item The optimal safety value function is defined as
            \begin{equation}
                \label{definition of optimal Qh}
                \begin{aligned}
                    Q_h^*(x, u, a)&=\max _{\pi_h} Q_h^{\pi_h}(x, u, a)\\&=\max _{\pi_h} \min _{\mu_h} \min _{\tau \sim \pi_h, \mu_h}\left\{Q_h^\tau(x, u, a)\right\}.
                \end{aligned}
            \end{equation}
    \end{enumerate}
\end{definition}

$Q_h^\tau(x, u, a)$ captures the most dangerous state in the trajectory, i.e., the state with the lowest constraint value. $Q_h^\tau(x, u, a) \geq 0$ implies that the whole trajectory $\tau$ satisfies the safety constraint $h(x)\geq 0$. $Q_h^{\pi_h, \mu_h}(x, u, a)$ represents the most dangerous state in the long term, when the system is driven by a specified pair of protagonist and adversary. Since both protagonist policy $\pi_h$ and adversary policy $\mu_h$ may be stochastic, $Q_h^{\pi_h, \mu_h}(x, u, a)$ captures the trajectory with the lowest safety value among all possible trajectories. $Q_h^{\pi_h}(x, u, a)$ identifies the safety-preserving capability of a specified protagonist policy $\pi_h$ under worst-case safety attacks. $Q_h^*(x, u, a)$ represents the best possible outcome we can get in terms of safety, starting from $(x,u,a)$. $Q_h^*(x, u, a)\leq 0$ implies that the system cannot guarantee safety at the state-action pair $(x,u,a)$. Therefore, $Q_h^*$ identifies the maximal robust invariant set of the system. We assume that the constrained zero-sum Markov game $\mathcal{M}=\left(\mathcal{X}, \mathcal{U}, \mathcal{A}, f, r, h, \gamma\right)$ satisfies that $\mathop{\max}\limits_{x\in \mathcal{X}}\mathop{\max}\limits_{u\in \mathcal{U}} \mathop{\min}\limits_{a\in \mathcal{A}} Q_h^{*}(x,u,a)\geq 0$, otherwise the system is impossible to achieve persistent safety. We formally define the robust invariant sets as follows.

\begin{definition}[constraint set]
    The constraint set is defined as the zero-superlevel set of the constraint function $h(x)$, i.e.,
    \begin{equation}
        \nonumber
        S_h=\left\{x\in \mathcal{X} \mathrel{}\middle|\mathrel{} h(x)\geq 0\right\}.
    \end{equation}
\end{definition}

\begin{definition}[robust invariant sets]
    \hfill
    \begin{enumerate}
        \item If we have $\mathop{\max}\limits_{x\in \mathcal{X}}\mathop{\max}\limits_{u\in \mathcal{U}} \mathop{\min}\limits_{a\in \mathcal{A}} Q_h^{\pi_h}(x,u,a)\geq 0$ for a policy $\pi_h$, the robust invariant set of this specified protagonist policy $\pi_h$ is defined as
            \begin{equation}
                \nonumber
                S_{\rm{r}}^{\pi_{h}}=\left\{x \in \mathcal{X}\mathrel{}\middle|\mathrel{} \max _{u \in \mathcal{U}} \min _{a\in \mathcal{A}} Q_h^{\pi_h}(x,u,a)\geq 0\right\}.
            \end{equation}
        \item The maximal robust invariant set is defined as
            \begin{equation}
                \nonumber
                S_{\rm{r}}^{*}=\left\{x \in \mathcal{X} \mathrel{}\middle|\mathrel{} \max _{u\in \mathcal{U}} \min _{a\in \mathcal{A}} Q_h^{*}(x,u,a)\geq 0\right\}.
            \end{equation}
    \end{enumerate}
\end{definition}

\begin{remark}
    Technically speaking, the robust invariant sets in this article are robust controlled invariant sets. For states inside these sets, there exists some control policy keeping the system safe regardless of any disturbances \cite{margellos2011hamilton}. The term controlled is omitted for simplicity.
\end{remark}

\begin{remark}
    \label{explanation for the necessity of using ris as constraint}
    We can easily deduce that the following set inclusion relationship holds:
        \begin{eqnarray}
            \label{set inclusion}
            S_{\rm{r}}^{\pi_{h}} \subseteq S_{\rm{r}}^{*} \subseteq S_h.
        \end{eqnarray}
    To ensure safety of the system, it is not enough to only constrain the state to stay inside the constraint set $S_h$. Once the system state enters the region $S_h \setminus S_{\rm{r}}^{*}$ (i.e., $h(x)\geq0$ and $\max \limits_{u\in \mathcal{U}} \min \limits_{a\in \mathcal{A}} Q_h^{*}(x,u,a)\leq 0$), there exists some adversary policy $\mu_h$ that will drive the system out of the constraint set $S_h$, regardless of what protagonist policy is adopted. Therefore, we must constrain the system to stay inside the maximal robust invariant set $S_{\rm{r}}^{*}$. To impose this constraint, firstly we must figure out a way to obtain $S_{\rm{r}}^{*}$, which is equivalent to solving for $Q_h^*(x,u,a)$.
\end{remark}

The set inclusion relationship (\ref{set inclusion}) suggests that we may perform policy iteration, which converges to an optimal protagonist policy $\pi_h^*$, such that $S_{\rm{r}}^{\pi_{h}^*} = S_{\rm{r}}^{*}$. First, we need the following proposition to simplify the policy iteration procedure.

\begin{proposition}
    \label{deterministic policy alternative}
    Given a pair of stochastic protagonist policy $\pi_h^{\rm{sto}}$ and adversary policy $\mu_h^{\rm{sto}}$, there exist a pair of deterministic protagonist policy $\pi_h^{\rm{det}}$ and adversary policy $\mu_h^{\rm{det}}$, such that $\forall x\in \mathcal{X},u\in \mathcal{U},a\in \mathcal{A}$,
    \begin{equation}
        \nonumber
        Q_{h}^{\pi_{h}^{\rm{sto}}, \mu_{h}^{\rm{sto}}}(x, u, a)=Q_h^{\pi_{h}^{\rm{det}}, \mu_{h}^{\rm{det}}}(x,u,a).
    \end{equation}
\end{proposition}

\begin{proof}
        Let $\tau_{\rm{wor}}$ denote the trajectory with the lowest safety values among all possible trajectories when the system starting from $(x,u,a)$ is driven by $\pi_h^{\rm{sto}}$ and $\mu_h^{\rm{sto}}$, i.e., $Q_h^{\pi_h, \mu_h}(x, u, a)=Q_h^{\tau_{\rm{wor}}}(x, u, a)$. Let $x_{\rm{wor}}$ denote the state with the lowest constraint value, i.e., $h(x_{\rm{wor}})=\mathop{\min}\limits_{x\in\tau_{\rm{wor}}}\left\{h(x)\right\}$. The key insight is that $x_{\rm{wor}}$ is reachable from $(x,u,a)$ under a pair of deterministic policies. Take the finite-horizon trajectory inside $\tau_{\rm{wor}}$ that starts at $(x,u,a)$ and ends at $x_{\rm{wor}}$. We denote this transition trajectory as $\tau_{\rm{trans}}$. If there are two identical states inside $\tau_{\rm{trans}}$, we remove the parts between the two states and the remaining trajectory still satisfies the system dynamics $x^{\prime}=f(x,u,a)$. Suppose that the filtered trajectory of $\tau_{\rm{trans}}$ is $\left\{x_0=x,u_0=u,a_0=a,x_1,u_1,a_1,x_2,u_2,a_2,\cdots\right\}$. The deterministic protagonist policy $\pi_h^{\rm{det}}$ is defined to satisfy
    \begin{equation}
        \nonumber
        \pi_h^{\rm{det}}(x_1)=u_1,\pi_h^{\rm{det}}(x_2)=u_2,\cdots.
    \end{equation}
    The deterministic adversary policy $\mu_h^{\rm{det}}$ is defined to satisfy
    \begin{equation}
        \nonumber
        \mu_h^{\rm{det}}(x_1)=a_1,\mu_h^{\rm{det}}(x_2)=a_2,\cdots.
    \end{equation}
    With this construction we have
    \begin{equation}
        \nonumber
        Q_{h}^{\pi_{h}^{\rm{sto}}, \mu_{h}^{\rm{sto}}}(x, u, a)=Q_h^{\pi_{h}^{\rm{det}}, \mu_{h}^{\rm{det}}}(x,u,a).
    \end{equation}
\end{proof}

Proposition \ref{deterministic policy alternative} indicates that we can restrict ourselves to using deterministic policies when solving for $Q_h^*$ without any loss of optimality. From now on, we assume that the safety protagonist policy $\pi_h$ and the safety adversary policy $\mu_h$ are deterministic. Since the system dynamics is also deterministic, the whole trajectory is fixed given the starting state-action pair $(x,u,a)$. The safety value function of a pair of $\pi_h$ and $\mu_h$ can be simplified to
\begin{equation}
    \nonumber
    \begin{gathered}
        Q_h^{\pi_h, \mu_h}(x, u, a)=\min _{t \in \mathbb{N}}\left\{h\left(x_t\right)\right\} \\
        \begin{aligned}
            \text { s.t. \quad }&x_0=x,u_0=u,a_0=a,\\
            &u_t=\pi_h\left(x_t\right), a_t=\mu_h\left(x_t\right),\\&x_t=f\left(x_{t-1}, u_{t-1}, a_{t-1}\right),t\geq1.
        \end{aligned}
    \end{gathered}
\end{equation}

Since the safety value functions are defined on infinite horizon, they naturally hold a recursive structure, which we refer to as self-consistency condition, just as the common value functions in RL. We have the following theorem.

\begin{theorem}[safety self-consistency conditions]
    \label{self-consistency conditions of safety value functions}
    The safety value functions satisfy the following self-consistency conditions:
    \begin{align}
        &\begin{aligned}
            \label{self-consistency condition of Qh pi mu}
            Q_h^{\pi_h, \mu_h}(x, u, a)=\min \left\{h(x), Q_h^{\pi_h, \mu_h}\left(x^{\prime}, \pi_h(x), \mu_h(x)\right)\right\},
        \end{aligned}\\
        &\begin{aligned}
            \label{self-consistency condition of Qh pi}
            Q_h^{\pi_h}(x, u, a)=\min \left\{h(x), \min _{a^{\prime} \in \mathcal{A}} Q_h^{\pi_h}\left(x^{\prime}, \pi_h(x), a^{\prime}\right)\right\},
        \end{aligned}\\
        &\begin{aligned}
            \label{self-consistency condition of optimal Qh}
            Q_h^*(x, u, a)=\min \left\{h(x), \max _{u^{\prime} \in \mathcal{U}} \min _{a^{\prime} \in \mathcal{A}} Q_h^*\left(x^{\prime}, u^{\prime}, a^{\prime}\right)\right\},
        \end{aligned}
    \end{align}
    in which $x^{\prime}=f(x,u,a)$.
\end{theorem}

\begin{proof}
    First, we prove (\ref{self-consistency condition of Qh pi mu}). From the definition of the safety value function, we have
    \begin{equation}
        \label{proof of safety self-consistency condition}
        \begin{aligned}
            Q_h^{\pi_h ,\mu_h}(x,u,a)&=\min _{t \geq 0}\left\{h\left(x_t\right)\right\}\\
            &=\min \left\{ h(x),\min_{t\ge 1} \left\{ h\left( x_t\right)\mathrel{}\middle|\mathrel{} x_1=x^{\prime}  \right\} \right\}\\
            &=\min \left\{ h(x),\min_{t\ge 0} \left\{ h\left( x_t\right)\mathrel{}\middle|\mathrel{} x_0=x^{\prime}  \right\} \right\}\\
            &=\min \left\{ h(x),Q_h^{\pi_h ,\mu_h}(x^{\prime},\pi_h(x),\mu_h(x))\right\},
        \end{aligned}
    \end{equation}
    in which $x^{\prime}=f(x,u,a)$.

    The proof for (\ref{self-consistency condition of Qh pi}) and (\ref{self-consistency condition of optimal Qh}) is similar to (\ref{proof of safety self-consistency condition}), with additional use of Bellman's principle of optimality.
\end{proof}

A problem with these self-consistency conditions is that they are not contraction mappings. To apply RL techniques such as policy iteration, we introduce their discounted versions, which are defined as follows.

\begin{definition}[safety operators]
    \label{safety operators}
    Given $\gamma_h\in (0,1)$, the safety self-consistency operator of a pair of protagonist policy $\pi_h$ and adversary policy $\mu_h$ is defined as
    \begin{equation}
        \label{safety self-consistency operator of pi mu}
        \begin{aligned}
            &[T_h^{\pi_h,\mu_h}(Q_h)](x,u,a)=(1-\gamma_h)h(x)\\&+\gamma_h \min \left\{h(x), Q_h\left(x^{\prime}, \pi_h(x), \mu_h(x)\right)\right\}.
        \end{aligned}
    \end{equation}
    The safety self-consistency operator of a protagonist policy $\pi_h$ is defined as
    \begin{equation}
        \label{safety self-consistency operator of pi}
        \begin{aligned}
            &[T_h^{\pi_h}(Q_h)](x,u,a)=(1-\gamma_h)h(x)\\&+\gamma_h \min \left\{h(x), \min _{a^{\prime} \in \mathcal{A}} Q_h\left(x^{\prime}, \pi_h(x), a^{\prime}\right)\right\}.
        \end{aligned}
    \end{equation}
    The safety Bellman operator is defined as
    \begin{equation}
        \label{safety Bellman operator}
        \begin{aligned}
            &[T_h(Q_h)](x,u,a)=(1-\gamma_h)h(x)\\&+\gamma_h \min \left\{h(x), \max _{u^{\prime} \in \mathcal{U}} \min _{a^{\prime} \in \mathcal{A}} Q_h\left(x^{\prime}, u^{\prime}, a^{\prime}\right)\right\}.
        \end{aligned}
    \end{equation}
    
\end{definition}

The following theorem indicates that the three safety operators in Definition \ref{safety operators} are monotone contractions.

\begin{theorem}[monotone contraction of safety operators]
    \label{monotone contraction of safety operators}
    \hfill
    \begin{enumerate}
        \item Given any $Q_h,\widetilde{Q}_h \in \mathbb{R}^{|\mathcal{X}|\cdot|\mathcal{U}|\cdot|\mathcal{A}|}$, we have
        \begin{equation}
            \nonumber
            \left\|T_h^{\pi_h,\mu_h}(Q_h)-T_h^{\pi_h,\mu_h}(\widetilde{Q}_h)\right\|_{\infty} \leq \gamma_h\left\|Q_h-\widetilde{Q}_h\right\|_{\infty},
        \end{equation}
        \begin{equation}
            \nonumber
            \left\|T_h^{\pi_h}(Q_h)-T_h^{\pi_h}(\widetilde{Q}_h)\right\|_{\infty} \leq \gamma_h\left\|Q_h-\widetilde{Q}_h\right\|_{\infty},
        \end{equation}
        \begin{equation}
            \nonumber
            \left\|T_h(Q_h)-T_h(\widetilde{Q}_h)\right\|_{\infty} \leq \gamma_h\left\|Q_h-\widetilde{Q}_h\right\|_{\infty}.
        \end{equation}
        \item Suppose $Q_h(x,u,a)\geq \widetilde{Q}_h(x,u,a)$ holds for all $x\in \mathcal{X},u\in \mathcal{U},a\in \mathcal{A}$, then we have
        \begin{equation}
            \nonumber
            [T_h^{\pi_h,\mu_h}(Q_h)](x,u,a)\geq [T_h^{\pi_h,\mu_h}(\widetilde{Q}_h)](x,u,a),
        \end{equation}
        \begin{equation}
            \nonumber
            [T_h^{\pi_h}(Q_h)](x,u,a)\geq [T_h^{\pi_h}(\widetilde{Q}_h)](x,u,a),
        \end{equation}
        \begin{equation}
            \nonumber
            [T_h(Q_h)](x,u,a)\geq [T_h(\widetilde{Q}_h)](x,u,a).
        \end{equation}
    \end{enumerate}
\end{theorem}

\begin{proof}
    We only prove the monotonicity and contraction of $T_h$. The proof for $T_h^{\pi_h}$ and $T_h^{\pi_h,\mu_h}$ is similar.

    Since the $\max$ and $\min$ operation contained in $T_h$ is monotone, $T_h$ is also monotone.
    Given $Q_h$ and $\widetilde{Q}_h$, we have
    \begin{equation}
        \nonumber
        \begin{aligned}
            &[T_h( Q_h )](x,u,a) - [T_h( \widetilde{Q}_h )](x,u,a)
            \\=&\gamma_h \min \left\{h(x), \max\limits_{u^{\prime}\in \mathcal{U}}\min\limits_{a^{\prime}\in \mathcal{A}}Q_h\left(x^{\prime},u^{\prime},a^{\prime}\right)\right\}\\&-\gamma_h \min \left\{h(x), \max\limits_{u^{\prime}\in \mathcal{U}}\min\limits_{a^{\prime}\in \mathcal{A}}\widetilde{Q}_h\left(x^{\prime},u^{\prime},a^{\prime}\right)\right\},
        \end{aligned}
    \end{equation}
    in which $x^{\prime}=f(x,u,a)$.
    Utilizing the relationship
    \begin{equation}
        \nonumber
        \begin{aligned}
            &\left|\max _x \min _y F(x, y)-\max _x \min _y G(x, y)\right| \\&\leq \max _x \max _y\left|F(x, y)-G(x, y)\right|,
        \end{aligned}
    \end{equation}
    we have
    \begin{equation}
        \nonumber
        \begin{aligned}
            &\left\|T_h( Q_h ) -T_h( \widetilde{Q}_h )\right\|_{\infty}\\
            \leq& \gamma_h \max\limits_{u^{\prime}\in \mathcal{U}} \max\limits_{a^{\prime}\in \mathcal{A}} \left| Q_h\left(x^{\prime},u^{\prime},a^{\prime}\right)-\widetilde{Q}_h\left(x^{\prime},u^{\prime},a^{\prime}\right) \right|\\
            \leq& \gamma_h \left\| Q_h - \widetilde{Q}_h \right\| _{\infty}.
        \end{aligned}
    \end{equation}
\end{proof}

Since the three operators in Definition \ref{safety operators} are contraction mappings, they all have unique fixed points. These fixed points serve as approximations of the original safety value functions in Definition \ref{safety value functions}. The following proposition shows that as the discount factor $\gamma_h$ goes to 1, these fixed points converge to the original safety values.

\begin{proposition}
    \label{convergence to the true safety values}
    As $\gamma_h$ goes to $1$, the fixed point of operator $T_h^{\pi_h,\mu_h}$ converges to the safety value function defined in (\ref{definition of Qh pi mu}), the fixed point of operator $T_h^{\pi_h}$ converges to the safety value function defined in (\ref{definition of Qh pi}), and the fixed point of operator $T_h$ converges to the safety value function defined in (\ref{definition of optimal Qh}).
\end{proposition}

\begin{proof}
    The key of proof is defining a discounted formulation of the safety values. Recall that $Q_h^{\pi_h, \mu_h}(x, u, a)=\min _{t \in \mathbb{N}}\left\{h\left(x_t\right)\right\}$, in which $h(x_t)$ denote the constraint values of the trajectory driven by $\pi_h$ and $\mu_h$. Define the following discounted value
    \begin{equation}
        \label{explicit form of safety value}
        \begin{aligned}
            D=&(1-\gamma_h) h(x_0)+\gamma_h {\left\{\operatorname {min} \left\{h(x_0),(1-\gamma_h) h(x_1)+\right.\right.} \\
            & \left.\gamma_h\left(\min \left\{h(x_1),(1-\gamma_h) h(x_2)+ \ldots\right)\right\}\right\}.
            \end{aligned}
    \end{equation}
    It can be verified that $D$ is the explicit formulation of the fixed point of $T_h^{\pi_h,\mu_h}$, i.e., $D=T_h^{\pi_h,\mu_h}(D)$. Taking the limit of (\ref{explicit form of safety value}) as $\gamma_h\rightarrow1$, we obatin
    \begin{equation}
        \nonumber
        \lim _{\gamma_h \rightarrow 1} D = \min _{t \in \mathbb{N}}\left\{h\left(x_t\right)\right\} = Q_h^{\pi_h, \mu_h}(x, u, a),
    \end{equation}
    which indicates that the fixed point of $T_h^{\pi_h,\mu_h}$ converges to the original safety value function defined in (\ref{definition of Qh pi mu}). The proof for $T_h^{\pi_h}$ and $T_h$ is similar.
\end{proof}

From now on, we assume that the chosen $\gamma_h$ is sufficiently close to $1$ and the fixed points of safety operators represent the original safety values.

\subsection{Problem Formulation}
In this subsection, we define a proper problem formulation for constrained zero-sum Markov games. Since safety is the top priority, the pursuit of high total rewards must be carried out under the condition that safety is guaranteed even with worst-case safety attacks. Given a state $x\in \mathcal{X}$, there are two cases. First, $x \notin S_{\rm{r}}^*$. The safety constraint $h(x)\geq 0$ will be violated sooner or later under the worst-case safety adversary. Therefore, it is pointless to optimize the total rewards of this kind of states. We should execute the safest action (i.e., optimize its safety value) and drive the system back to $S_{\rm{r}}^*$ as soon as possible. Second, $x \in S_{\rm{r}}^*$. We should maximize the total rewards and ensure persistent safety (i.e., do not choose those actions that could drive the system out of $S_{\rm{r}}^*$). We have the following definition and proposition.

\begin{definition}[invariant policy sets]
    \hfill
    \begin{enumerate}
        \item The invariant policy set $\Pi_{\rm{s}}^{\pi_h}$ specified by a robust invariant set $S_{\rm{r}}^{\pi_h}$ is defined as
        \begin{equation}
            \nonumber
            \Pi_{\rm{s}}^{\pi_h}=\left\{\pi \mathrel{}\middle|\mathrel{} \forall x\in S_{\rm{r}}^{\pi_h}, \pi(u\mid x)=0, \text{if } u\notin U_{\rm{s}}^{\pi_h}(x)\right\},
        \end{equation}
        in which $U^{\pi_h}_{\rm{s}}(x)$ contains the admissible action to maintain persistent safety at state $x\in S_{\rm{r}}^{\pi_h}$, i.e.,
        \begin{equation}
            \nonumber
            U_{\rm{s}}^{\pi_h}(x)=\left\{u \in \mathcal{U} \mathrel{}\middle|\mathrel{} \min _{a\in \mathcal{A}} Q_h^{\pi_h}(x,u,a)\geq 0\right\}.
        \end{equation}
        \item The optimal invariant policy set $\Pi_{\rm{s}}^{*}$ is defined as
        \begin{equation}
            \nonumber
            \Pi_{\rm{s}}^{*}=\left\{\pi \mathrel{}\middle|\mathrel{} \forall x\in S_{\rm{r}}^*, \pi(u\mid x)=0, \text{if } u\notin U_{\rm{s}}^*(x)\right\},
        \end{equation}
        in which $U^*_{\rm{s}}(x)$ contains the admissible action to maintain persistent safety at state $x\in S_{\rm{r}}^*$, i.e.,
        \begin{equation}
            \nonumber
            U_{\rm{s}}^*(x)=\left\{u \in \mathcal{U} \mathrel{}\middle|\mathrel{} \min _{a\in \mathcal{A}} Q_h^{*}(x,u,a)\geq 0\right\}.
        \end{equation}
    \end{enumerate}
\end{definition}

\begin{proposition}
    \hfill
    \label{forward invariance under invariant policy sets}
    \begin{enumerate}
        \item Given $x\in S_{\rm{r}}^{\pi_h}$, if the protagonist policy $\pi \in \Pi_{\rm{s}}^{\pi_h}$, the system will stay in $S_{\rm{r}}^{\pi_h}$ no matter what adversary policy $\mu$ is adopted.
        \item Given $x\in S_{\rm{r}}^{*}$, if the protagonist policy $\pi \in \Pi_{\rm{s}}^{*}$, the system will stay in $S_{\rm{r}}^{*}$ no matter what adversary policy $\mu$ is adopted.
    \end{enumerate}
\end{proposition}

\begin{proof}
    We only prove the first claim in Proposition \ref{forward invariance under invariant policy sets}. The proof for the second claim is similar.

    Given $x\in S_{\rm{r}}^{\pi_h}$, if the protagonist policy $\pi \in \Pi_{\rm{s}}^{\pi_h}$, we have
    \begin{equation}
        \nonumber
        Q_h^{\pi_h}(x,u,a)\geq 0, u\sim \pi(\cdot\mid x), \forall a\in \mathcal{A}.
    \end{equation}
    Utilizing the self-consistency condition (\ref{self-consistency condition of Qh pi}) for $Q_h^{\pi_h}$, we have
    \begin{equation}
        \nonumber
        Q_h^{\pi_h}(x,u,a) = \min \left\{h(x), \mathop{\min} \limits_{a^{\prime} \in \mathcal{A}} Q_h^{\pi_h}\left(x^{\prime}, \pi_h(x), a^{\prime}\right)\right\}.
    \end{equation}
    Therefore $\mathop{\min} \limits_{a^{\prime} \in \mathcal{A}} Q_h^{\pi_h}\left(x^{\prime}, \pi_h(x), a^{\prime}\right)\geq 0$ holds, which implies that $\mathop{\max} \limits_{u^{\prime} \in \mathcal{U}}\mathop{\min} \limits_{a^{\prime} \in \mathcal{A}} Q_h^{\pi_h}\left(x^{\prime}, u^{\prime}, a^{\prime}\right)\geq 0$. Based on the definition of $S_{\rm{r}}^{\pi_h}$, we have $x^{\prime}\in S_{\rm{r}}^{\pi_h}$. Utilizing this procedure recursively, we conclude that starting from $x\in S_{\rm{r}}^{\pi_h}$, if the protagonist policy $\pi \in \Pi_{\rm{s}}^{\pi_h}$, the system state will always be in $S_{\rm{r}}^{\pi_h}$ regardless of what adversary actions are taken.
\end{proof}

Proposition \ref{forward invariance under invariant policy sets} indicates that any policy belonging to an invariant policy set will keep the system inside the corresponding robust invariant set, thus achieving persistent safety. Essentially, the infinite-horizon safety constraint $h(x_t)\geq 0,t\in\mathbb{N}$ is transformed into state-dependent constraints on the action space $\mathcal{U}$ specified by invariant policy sets. To capture the performance optimality inside the robust invariant sets, we define the following induced zero-sum Markov games.

\begin{definition}[induced zero-sum Markov games]
    \hfil
    \begin{enumerate}
        \item Given the original constrained zero-sum Markov game $\mathcal{M}=\left(\mathcal{X}, \mathcal{U}, \mathcal{A}, f, r, h, \gamma\right)$ and a robust invariant set $S_{\rm{r}}^{\pi_h}$, the induced zero-sum Markov game is defined as $\mathcal{M}^{\pi_h}=\left(S_{\rm{r}}^{\pi_h}, \mathop{\bigcup}\limits_{x}{U_{\rm{s}}^{\pi_h}(x)}, \mathcal{A}, f, r, \gamma\right)$.
        \item Given the original constrained zero-sum Markov game $\mathcal{M}=\left(\mathcal{X}, \mathcal{U}, \mathcal{A}, f, r, h, \gamma\right)$ and the optimal robust invariant set $S_{\rm{r}}^{*}$, the optimal induced zero-sum Markov game is defined as $\mathcal{M}^{*}=\left(S_{\rm{r}}^{*}, \mathop{\bigcup}\limits_{x}{U_{\rm{s}}^{*}(x)}, \mathcal{A}, f, r, \gamma\right)$.
    \end{enumerate}
\end{definition}

The notation $\mathop{\bigcup}\limits_{x}{U_{\rm{s}}^{\pi_h}(x)}$ means that for each state $x$, the admissible state-dependent action space is $U_{\rm{s}}^{\pi_h}(x)$. Based on Proposition \ref{forward invariance under invariant policy sets}, it is easy to check that these games are well-defined. Also note that the induced zero-sum Markov games are unconstrained. As discussed before, it is only meaningful to optimize total rewards inside the maximal robust invariant set. We present the definitions of value functions as follows.

\begin{definition}[value functions]
    \hfill
    \label{value functions}
    \begin{enumerate}
        \item The value function of a protagonist policy $\pi$ and an adversary policy $\mu$ is defined as
            \begin{equation}
                \label{definition of Q pi mu}
                Q^{\pi, \mu}(x, u, a)=\sum_{t=0}^{\infty} \gamma^t r(x_t,u_t,a_t),
            \end{equation}
        in which $x_0=x,u_0=u,a_0=a$ and for $t\geq1$, $x_t=f(x_{t-1},u_{t-1},a_{t-1}),u_t\sim \pi \left( \cdot \mid x_t \right) , a_t\sim \mu \left( \cdot \mid x_t \right)$.
        \item The value function of a protagonist policy $\pi$ is defined as
            \begin{equation}
                \label{definition of Q pi}
                Q^{\pi}(x, u, a)=\min _{\mu} Q^{\pi, \mu}(x, u, a).
            \end{equation}
        \item For the optimal induced zero-sum Markov game $\mathcal{M}^{*}=\left(S_{\rm{r}}^{*}, \mathop{\bigcup}\limits_{x}{U_{\rm{s}}^{*}(x)}, \mathcal{A}, f, r, \gamma\right)$, the optimal value function is defined as
            \begin{equation}
                \label{definition of optimal Q}
                \begin{aligned}
                    Q^{*}(x, u, a)&=\max _{\pi \in \Pi_{\rm{s}}^*} Q^{\pi}(x,u,a)\\&=\max _{\pi \in \Pi_{\rm{s}}^*}\min _{\mu} Q^{\pi, \mu}(x, u, a).
                \end{aligned}
            \end{equation}
    \end{enumerate}
\end{definition}

As shown in (\ref{definition of optimal Q}), the optimal value function is defined for states inside the maximal robust invariant set $S_{\rm{r}}^*$ (i.e., for the optimal induced zero-sum Markov games) and the optimal policy is searched inside the optimal invariant policy set $\Pi_{\rm{s}}^*$, which contains all the policies that are able to maintain persistent safety of the system under worst-case safety attacks. Let $d$ denote the initial state distribution. Based on the definitions of value functions and safety value functions, we formally specify the problem formulation for constrained zero-sum Markov games, in which the objective for protagonist is twofold.

\begin{equation}
    \label{problem formulation for constrained zero-sum Markov games}
    \begin{gathered}
        \max _\pi \mathop{\mathbb{E}}\limits_{x_0\sim d}\left\{V^{\pi}(x_0)\cdot \mathbb{1}_{S_{\rm{r}}^*}(x_0)+V_h^{\pi}(x_0)\cdot \mathbb{1}_{\mathcal{X}\setminus S_{\rm{r}}^*}(x_0)\right\} \\
        \begin{aligned}
            \text { s.t. \quad }&x_{t+1}=f\left(x_t, u_t, a_t\right), u_t\sim \pi\left( \cdot \mid x_t \right), t\geq0,\\& \min _{a\in \mathcal{A}} Q_h^*\left(x_t, u_t, a\right) \geq 0, \ \forall x_t\in S_{\rm{r}}^*.
        \end{aligned}
    \end{gathered}
\end{equation}
$Q_h^*$ represents the fixed point of the safety Bellman operator $T_h$, as defined in (\ref{safety Bellman operator}). $\mathbb{1}_{S_{\rm{r}}^*}$ represents the indicator function, i.e., $\mathbb{1}_{S_{\rm{r}}^*}(x)=1$ when $x\in S_{\rm{r}}^*$ and otherwise $\mathbb{1}_{S_{\rm{r}}^*}(x)=0$. $V^{\pi}$ and $V_h^{\pi}$ represent the state-value function and the safety state-value function, i.e.,
\begin{equation}
    \nonumber
    V^{\pi}(x)=\sum_{u\in \mathcal{U}}\pi(u\mid x) \min_{a\in \mathcal{A}} Q^{\pi}(x,u,a),
\end{equation}
\begin{equation}
    \nonumber
    V_h^{\pi}(x)=\max _{u\in \mathcal{U}} \min _{a\in \mathcal{A}} Q_h^{\pi}(x,u,a).
\end{equation}

The value functions in Definition \ref{value functions} are defined on infinite horizon, so they naturally hold a recursive structure, which we refer to as the self-consistency condition. We present the self-consistency conditions of value functions in the following theorem and define the corresponding operators.

\begin{theorem}[performance self-consistency conditions]
    \label{self-consistency conditions of value functions}
    The value functions satisfy the following self-consistency conditions:
    \begin{align}
        &\begin{aligned}\label{self-consistency condition of Q pi mu}
            &Q^{\pi ,\mu}(x,u,a)=r(x,u,a)\\&+\gamma\sum_{u^{\prime}\in \mathcal{U}}{\pi}\left( u^{\prime}\mid x^{\prime} \right) \sum_{a^{\prime}\in \mathcal{A}}{\mu}\left( a^{\prime}\mid x^{\prime} \right) Q^{\pi ,\mu}\left( x^{\prime},u^{\prime},a^{\prime} \right),
        \end{aligned}\\
        &\begin{aligned}\label{self-consistency condition of Q pi}
            &Q^{\pi}(x,u,a)=r(x,u,a)\\&+\gamma\sum_{u^{\prime}\in \mathcal{U}}{\pi}\left( u^{\prime}\mid x^{\prime} \right) \min_{a^{\prime}\in \mathcal{A}} Q^{\pi}\left( x^{\prime},u^{\prime},a^{\prime} \right),
        \end{aligned}\\
        &\begin{aligned}\label{self-consistency condition of optimal Q}
            &Q^*(x,u,a)=r(x,u,a)\\&+\gamma \max_{\pi(\cdot\mid x^{\prime}) \in \Pi_{\rm{s}}^*}\sum_{u^{\prime}\in \mathcal{U}}{\pi}\left( u^{\prime}\mid x^{\prime} \right) \min_{a^{\prime}\in \mathcal{A}} Q^*\left( x^{\prime},u^{\prime},a^{\prime} \right),
        \end{aligned}
    \end{align}
    in which $x^{\prime}=f(x,u,a)$. (\ref{self-consistency condition of optimal Q}) only applies to the optimal induced zero-sum Markov game $\mathcal{M}^*$.
\end{theorem}

\begin{proof}
    (\ref{self-consistency condition of Q pi mu}) and (\ref{self-consistency condition of Q pi}) are the same as the standard self-consistency conditions in zero-sum Markov games \cite{perolat2015approximate}. Since the optimal induced zero-sum Markov game $\mathcal{M}^*$ is well-defined, (\ref{self-consistency condition of optimal Q}) follows from considering the Bellman equation on $\mathcal{M}^*$.
\end{proof}

\begin{definition}[performance operators]
    \label{performance operators}
    The self-consistency operator of a pair of protagonist policy $\pi_h$ and adversary policy $\mu_h$ is defined as
    \begin{equation}
        \label{performance self-consistency operator of pi mu}
        \begin{aligned}
            &[T^{\pi ,\mu}(Q)](x,u,a)=r(x,u,a)\\&+\gamma\sum_{u^{\prime}\in \mathcal{U}}{\pi}\left( u^{\prime}\mid x^{\prime} \right) \sum_{a^{\prime}\in \mathcal{A}}{\mu}\left( a^{\prime}\mid x^{\prime} \right) Q\left( x^{\prime},u^{\prime},a^{\prime} \right).
        \end{aligned}
    \end{equation}
    The self-consistency operator of a protagonist policy $\pi_h$ is defined as
    \begin{equation}
        \label{performance self-consistency operator of pi}
        \begin{aligned}
            &[T^{\pi}(Q)](x,u,a)=r(x,u,a)\\&+\gamma\sum_{u^{\prime}\in \mathcal{U}}{\pi}\left( u^{\prime}\mid x^{\prime} \right) \min_{a^{\prime}\in \mathcal{A}} Q\left( x^{\prime},u^{\prime},a^{\prime} \right).
        \end{aligned}
    \end{equation}
    The Bellman operator on the optimal induced zero-sum Markov game $\mathcal{M}^*$ is defined as
    \begin{equation}
        \label{constrained Bellman operator}
        \begin{aligned}
            &[T(Q)](x,u,a)=r(x,u,a)\\&+\gamma \max_{\pi(\cdot\mid x^{\prime}) \in \Pi_{\rm{s}}^*}\sum_{u^{\prime}\in \mathcal{U}}{\pi}\left( u^{\prime}\mid x^{\prime} \right) \min_{a^{\prime}\in \mathcal{A}} Q\left( x^{\prime},u^{\prime},a^{\prime} \right).
        \end{aligned}
    \end{equation}
    
\end{definition}

\subsection{Dual Policy Iteration}

A key issue of problem (\ref{problem formulation for constrained zero-sum Markov games}) is that we do not know the maximal robust invariant set $S_{\rm{r}}^*$ or the optimal safety value function $Q_h^*$ in advance. To address this challenge, we propose a dual policy iteration scheme, which simultaneously optimizes two policies: the task policy and the safety policy. The safety policy seeks the highest safety values. The task policy seeks the highest rewards inside the robust invariant set specified by the safety policy and seeks the highest safety value outside the robust invariant set. The pseudo-code of dual policy iteration is presented in Algorithm \ref{dual policy iteration}. Note that for the first time of task policy evaluation, we check that $\mathop{\max}\limits_{x\in \mathcal{X}}\mathop{\max}\limits_{u\in \mathcal{U}} \mathop{\min}\limits_{a\in \mathcal{A}} Q_h^{\pi_h}(x,u,a)\geq 0$ holds, otherwise go back to safety policy evaluation.

\begin{algorithm}
    \label{dual policy iteration}
    \caption{Dual Policy Iteration}
    \KwIn{initial task policy $\pi$, initial safety policy $\pi_h$.}
    \For{$m$ times}{

        \For{$n$ times}{
            (safety policy evaluation)

            Solve for $Q_h^{\pi_h}$ such that $T_h^{\pi_h}(Q_h^{\pi_h})=Q_h^{\pi_h}$.

            (safety policy improvement)

            \For{each $x \in \mathcal{X}$}{
                $\pi_h (x)\leftarrow \mathop{\argmax}\limits_{u\in \mathcal{U}}\left\{\mathop{\min}\limits_{a \in \mathcal{A}} Q_h^{\pi_h}(x,u,a)\right\}$.
            }
        }

        (task policy evaluation)

        Solve for $Q^{\pi}$ such that $T^{\pi}(Q^{\pi})=Q^{\pi}$.

        (task policy improvement)

        \For{each $x \in S_{\rm{r}}^{\pi_h}$}{
            
            \begin{flalign*}
                \hspace{-1.5mm}&\pi (\cdot\mid x)\leftarrow \mathop{\argmax}\limits_{\pi (\cdot\mid x)\in \Pi_{\rm{s}}^{\pi_{h}}}\mathop{\min}\limits_{a \in \mathcal{A}}\mathop{\sum}\limits_{u\in \mathcal{U}}\pi(u\mid x)Q^{\pi}(x,u,a).&
            \end{flalign*}
        }

        \For{each $x \notin S_{\rm{r}}^{\pi_h}$}{
            $\pi (x)\leftarrow \pi_h(x)$.
        }
    }
\end{algorithm}

\begin{remark}
    \label{matrix game}
    The task policy improvement on state $x \in S_{\rm{r}}^{\pi_h}$ is equivalent to the following linear programming problem:
    \begin{equation}
        \label{linear programming formulation}
        \begin{gathered}
            \max _{(\pi(\cdot\mid x),c)} \ c \\
            \begin{aligned}
                \text { s.t. \quad }&\mathop{\sum}\limits_{u\in \mathcal{U}}\pi(u\mid x)Q^{\pi}(x,u,a)\geq c,\ \forall a\in\mathcal{A},\\& \mathop{\sum}\limits_{u\in \mathcal{U}}\pi(u\mid x)=1,\\& \pi(u\mid x)\geq0, \ \forall u\in U_{\rm{s}}^{\pi_h}(x),\\& \pi(u\mid x)=0, \ \forall u\notin U_{\rm{s}}^{\pi_h}(x).
            \end{aligned}
        \end{gathered}
    \end{equation}
    Note that $Q^{\pi}$ in (\ref{linear programming formulation}) is a constant, which is obtained by previous task policy evaluation.

\end{remark}

\begin{remark}
    In dual policy iteration, the task policy can be stochastic while the safety policy is restricted to the class of deterministic policies. As shown in Proposition \ref{deterministic policy alternative}, only considering deterministic policies when solving for $Q_h^*$ does not lose any optimality. In task policy improvement, for states inside the robust invariant set, it is equivalent to finding the Nash equilibrium of a matrix game, as shown in Remark \ref{matrix game}, which requires the task policy to be stochastic to guarantee the existence of a Nash equilibrium. For states outside the robust invariant set, we only optimize their safety values, so we directly copy the deterministic safety policy to the task policy: $\pi (x)\leftarrow \pi_h(x)$.
\end{remark}

\begin{remark}
    The hyperparameter $n$ in Algorithm \ref{dual policy iteration} controls the relative update frequency between the task policy and the safety policy. An extreme case is that by choosing an $n$ sufficiently large, we have optimal safety value $Q_h^*$ and the maximal robust invariant set $S_{\rm{r}}^*$ precomputed before optimizing the task policy. This case works well in the tabular setting. However, its corresponding deep RL algorithm suffers from the distribution mismatch issue, i.e., the pretrained $Q_h$ and the task policy $\pi$ are on different data basis, which may lead to poor algorithm performance.
\end{remark}

We prove that both policies converge to the optimal ones in the following theorem.

\begin{theorem}[convergence of dual policy iteration]
    \label{convergence of dual policy iteration}
    By choosing an $m$ sufficiently large in Algorithm \ref{dual policy iteration}, the task policy $\pi$ converges to the solution of problem (\ref{problem formulation for constrained zero-sum Markov games}) and the safety value $Q_h^{\pi_h}$ of the safety policy $\pi_h$ converges to the fixed point of safety Bellman operator $T_h$.
\end{theorem}

\begin{proof}
    First we prove the convergence of safety policy $\pi_h$. $Q_h$ is regarded as a vector in Euclidean space, i.e., $Q_h \in \mathbb{R}^{|\mathcal{X}|\cdot|\mathcal{U}|\cdot|\mathcal{A}|}$. Let $k$ denote the iteration number of the safety policy.
We set out to prove the following relationship:
\begin{equation}
    \label{recursion for monotone convergence}
    Q_{h}^{\pi_h ^k}\le T_h\left( Q_{h}^{\pi_h ^k} \right) \le Q_{h}^{\pi_h ^{k+1}}\le Q_{h}^{*}.
\end{equation}
Based on the definition of $T_h$ and safety policy improvement, we have
\begin{equation}
    \nonumber
    T_{h} \left( Q_h^{\pi_h ^k}\right)=T_h^{\pi_h^{k+1}}\left(Q_h^{\pi_h ^k}\right) \geq T_h^{\pi_h ^k}\left(Q_h^{\pi_h ^k}\right)=Q_h^{\pi_h ^k}.
\end{equation}
Since $Q_h^{\pi_h ^k} \leq T_h^{\pi_h^{k+1}}\left(Q_h^{\pi_h ^k}\right)$, using the monotone contraction of $T_h^{\pi_h^{k+1}}$, we have
\begin{equation}
    \nonumber
    \begin{aligned}
        Q_h^{\pi_h ^k} &\leq T_h\left(Q_h^{\pi_h ^k}\right)=T_h^{\pi_h^{k+1}}\left(Q_h^{\pi_h ^k}\right) \\&\leq\left(T_h^{\pi_h^{k+1}}\right)^{\infty}\left(Q_h^{\pi_h ^k}\right)=Q_h^{\pi_h^{k+1}}.
    \end{aligned}
\end{equation}
Since $T_h\left(Q_h^{\pi_h^{k+1}}\right) \geq T_h^{\pi_h^{k+1}}\left(Q_h^{\pi_h^{k+1}}\right)=Q_h^{\pi_h^{k+1}}$, using the monotone contraction of $T_h$, we obtain
\begin{equation}
    \nonumber
    Q_h^*=(T_h)^{\infty}\left(Q_h^{\pi_h^{k+1}}\right) \geq \cdots \geq Q_h^{\pi_h^{k+1}}.
\end{equation}
So we conclude that (\ref{recursion for monotone convergence}) holds. The safety value sequence $\left\{Q_h^{\pi_h ^k}\right\}$ is monotone and bounded, so it converges. After the safety policy convergences, we have $Q_h^{\pi_h ^k}=T_h\left(Q_h^{\pi_h ^k}\right)=Q_h^{\pi_h^{k+1}}$, indicating that the sequence $\left\{Q_h^{\pi_h ^k}\right\}$ converges to the fixed point of $T_h$. The monotonicity of the safety value sequence also implies that the robust invariant set $S_r^{\pi_h}$ is expanding.

Next we prove the convergence of task policy $\pi$. Let $j$ denote the iteration number of the task policy. The task policy improvement for $x\in S_{\rm{r}}^{\pi_h}$ is equivalent to a standard policy improvement on the induced zero-sum Markov game $\mathcal{M}^{\pi_h}=\left(S_{\rm{r}}^{\pi_h}, \mathop{\bigcup}\limits_{x}{U_{\rm{s}}^{\pi_h}(x)}, \mathcal{A}, f, r, \gamma\right)$. Therefore, we have
\begin{equation}
    \nonumber
    Q^{\pi_{j+1}}(x,u,a)\geq Q^{\pi_{j}}(x,u,a), \forall x\in S_{\rm{r}}^{\pi_h}, u\in U_{\rm{s}}^{\pi_h}(x), a\in\mathcal{A}.
\end{equation}
As the iteration of safety policy $\pi_h$ goes on, $S_{\rm{r}}^{\pi_h}$ and the corresponding $U_{\rm{s}}^{\pi_h}(x)$ are expanding and converging to $S_{\rm{r}}^{*}$ and $U_{\rm{s}}^{*}(x)$. We can conclude that the value function $Q^{\pi}$ on $S_{\rm{r}}^{*}$ will converge to the fixed point of the Bellman operator on the optimal induced zero-sum Markov game $\mathcal{M}^*$ (\ref{constrained Bellman operator}). For states outside the maximal robust invariant set, i.e., $x\notin S_{\rm{r}}^*$, the task policy directly copies the safety policy, so its convergence is the same as the safety policy. In summary, for states inside the maximal robust invariant set, the converged task policy seeks the highest total rewards under the condition of persistent safety is guaranteed; for states outside the maximal robust invariant set, the converged task policy seeks the highest constraint values as the converged safety policy does. Therefore, the iteration of task policy converges to the solution of problem (\ref{problem formulation for constrained zero-sum Markov games}).
\end{proof}

\begin{remark}
    We highlight that during the iteration, the robust invariant set of the task policy is also non-shrinking. Since the task policy improvement searches for a new task policy $\pi$ in the invariant policy set $\Pi_{\rm{s}}^{\pi_h}$, the robust invariant set $S_{\rm{r}}^{\pi}$ of the task policy $\pi$ is at least the same size as $S_{\rm{r}}^{\pi_h}$, i.e., $x \in S_{\rm{r}}^{\pi_h} \rightarrow x\in S_{\rm{r}}^{\pi}$. This merit contributes to the stability of the training process of our proposed DRAC.
\end{remark}

\begin{remark}
    Our proposed dual policy iteration also yields a sound solution for the common safe RL problem, when disturbances do not exist. The robust invariant sets degenerate to the standard invariant sets. Both safety policy and task policy also converge to the optimal ones in the no-disturbance case.
\end{remark}

\section{Dually Robust Actor-Critic}

Based on the proposed dual policy iteration scheme, in this section, we present dually robust actor-critic (DRAC), a deep RL algorithm that can learn one policy that is safe and simultaneously robust to both performance and safety attacks. Our algorithm is built on top of soft actor-critic (SAC) \cite{haarnoja2018soft}, a well-known model-free off-policy RL algorithm.

We denote the task policy network as $\pi(x;\theta)$ and the safety policy network as $\pi_h(x;\phi)$. Since the $\mathop{\min}\limits_{\mu}$ operations in both task policy evaluation and safety policy evaluation are hard to conduct for continuous state and action spaces, we train two adversary networks: the performance adversary network $\mu(x;\beta)$ and the safety adversary network $\mu_h(x;\xi)$. The task policy and performance adversary are stochastic, while the safety policy and safety adversary are deterministic. We follow the double Q-network design in SAC and make use of two performance value networks, denoted as $Q(x,u,a;\omega_1)$ and $Q(x,u,a;\omega_2)$. The safety value network is denoted as $Q_h(x,u,a;\psi)$.

For a set $\mathcal{D}$ of collected samples, the loss functions of performance value networks are
\begin{equation}
    \nonumber
    L_Q\left(\omega_i\right)=\mathbb{E}_{\left(x, u, a, r, x^{\prime}\right) \sim \mathcal{D}}\left\{\left(Q(x,u,a;\omega_i)-\hat{Q}\right)^2\right\},
\end{equation}
where $i={1,2}$ and
\begin{equation}
    \nonumber
    \hat{Q}=r(x, u, a)+\gamma\left( Q(x^{\prime},u^{\prime}, a^{\prime};\hat{\omega}_j)-\alpha \log \pi\left(u^{\prime}|x^{\prime};\theta\right)\right),
\end{equation}
in which $j$ is randomly chosen from $\left\{1,2\right\}$, $\hat{\omega}_j$ denotes the target nework parameters, $u^{\prime}\sim \pi(\cdot\mid x^{\prime};\theta)$, $a^{\prime} = \mu(\cdot\mid{x^{\prime};\beta})$ and $\alpha$ denotes the temperature. The loss function for the temperature $\alpha$ is
\begin{equation}
    \nonumber
    L(\alpha)=\mathbb{E}_{x \sim \mathcal{D}}\left\{-\alpha \log \pi(u | x;\theta)-\alpha \mathcal{H}\right\},
\end{equation}
in which $\mathcal{H}$ represents the target entropy and $u\sim \pi(\cdot\mid x;\theta)$. The loss function of safety value network is
\begin{equation}
    \nonumber
    L_{Q_h}(\psi) = \mathbb{E}_{(x,u,a,h,x^{\prime}) \sim \mathcal{D}}\left\{\left(Q_h(x,u,a ; \psi)-\hat{Q}_h\right)^2\right\},
\end{equation}
where
\begin{equation}
    \nonumber
    \hat{Q}_h = (1-\gamma_h)h(x) + \gamma_h \min \left\{h(x), Q_h\left(x^{\prime}, u^{\prime}, a^{\prime};\hat{\psi}\right)\right\},
\end{equation}
in which $u^{\prime}= \pi_h(x^{\prime};\phi)$, $a^{\prime}= \mu_h({x^{\prime};\xi})$ and $\hat{\psi}$ denote the target network parameters. The loss functions of the safety policy and the safety adversary are
\begin{equation}
    \nonumber
    L_{\pi_h}(\phi)=-\mathbb{E}_{x \sim \mathcal{D}}\left\{Q_h\left(x,u,a ; \psi\right)\right\},
\end{equation}
\begin{equation}
    \nonumber
    L_{\mu_h}(\xi)=\mathbb{E}_{x \sim \mathcal{D}}\left\{Q_h\left(x,u,a ; \psi\right)\right\},
\end{equation}
in which $u= \pi_h(x;\phi)$ and $a= \mu_h({x;\xi})$. The loss function of the performance adversary is
\begin{equation}
    \nonumber
    L_{\mu}(\beta)=\mathbb{E}_{x \sim \mathcal{D}}\left\{Q\left(x,u,a ; \omega_i\right)\right\},
\end{equation}
in which $u\sim \pi(\cdot\mid x;\phi)$, $a\sim \mu({\cdot\mid x;\xi})$ and $j$ is randomly chosen from $\left\{1,2\right\}$.

To ensure persistent safety under worst-case safety attacks, the task policy $\pi$ must belong to the invariant policy set $\Pi_{\rm{s}}^{\pi_h}$ identified by the safety value function $Q_h$, i.e., $Q_h(x,u,a;\psi)\geq 0$ for $u\sim \pi(\cdot\mid x;\phi)$ and $a= \mu_h({x;\xi})$. We utilize the method of Lagrange multipliers to carry out the constrained policy optimization on $\pi$. For continuous state and action spaces, there are infinite constraints on the task policy $\pi$, so we adopt a Lagrange multiplier network $\lambda(x;\zeta)$ to facilitate the learning process \cite{narasimhan2020approximate}. The Lagrangian is formulated as
\begin{equation}
    \label{Lagrangian}
    \mathcal{L}(\theta, \zeta)=\mathbb{E}_{x \in \mathcal{X}}\left\{Q(x,u,a_1;\omega_j)+\lambda(x;\zeta) Q_h(x,u,a_2;\psi)\right\},
\end{equation}
in which $u\sim \pi(\cdot\mid x;\theta)$, $a_1\sim \mu({x;\beta})$, $a_2= \mu_h({x;\xi})$ and $j$ is randomly chosen from $\left\{1,2\right\}$. We solve for the saddle point of the Lagrangian using dual ascent. For states outside the maximal robust invariant set $S_{\rm{r}}^*$, their safety values are always smaller than zero. Therefore, the corresponding Lagrange multipliers $\lambda(x;\zeta)$ will go to $+\infty$. In this case, the second term becomes dominant in (\ref{Lagrangian}), which indicates that we only optimize the safety value functions for states outside $S_{\rm{r}}^*$, as in dual policy iteration. For practical implementations, we set an upper bound $\lambda_{max}$ for the outputs of $\lambda(x;\zeta)$. The loss function of the task policy is
\begin{equation}
    \label{task policy loss}
    \begin{aligned}
        L_{\pi}(\theta)=&\mathbb{E}_{x \sim \mathcal{D}}\left\{\alpha \log \pi(u \mid x;\theta)-Q(x, u, a_1;\omega_j)\right\}\\&-\mathbb{E}_{x\in\mathcal{D}}\left\{\lambda(x;\zeta)Q_h(x,u,a_2;\psi)\right\},
    \end{aligned}
\end{equation}
in which $u\sim \pi(\cdot\mid x;\theta)$, $a_1\sim \mu({x;\beta})$, $a_2= \mu_h({x;\xi})$ and $j$ is randomly chosen from $\left\{1,2\right\}$. The loss function of the Lagrange multiplier network is composed of two parts, i.e., $L_{\lambda}(\zeta)=L^{A}_{\lambda}(\zeta)+L^{B}_{\lambda}(\zeta)$. For state $x$ inside the current robust invariant set $S_{\rm{r}}^{\pi_h}$, i.e., $Q_h(x,\pi_h(x;\phi),\mu_h(x;\xi);\psi)\geq0$, the loss function is
\begin{equation}
    \label{first part of Lagrangian loss}
    L^{A}_{\lambda}(\zeta)=\mathbb{E}_{x\in\mathcal{D}}\left\{\lambda(x;\zeta)Q_h(x,u,a_2;\psi)\right\}.
\end{equation}
For state $x$ outside the current robust invariant set $S_{\rm{r}}^{\pi_h}$, i.e., $Q_h(x,\pi_h(x;\phi),\mu_h(x;\xi);\psi)<0$, the loss function is
\begin{equation}
    \label{second part of Lagrangian loss}
    L^{B}_{\lambda}(\zeta)=\mathbb{E}_{x\in\mathcal{D}}\left\{(\lambda(x;\zeta)-\lambda_{max})^2\right\}.
\end{equation}
We should focus on optimizing the safety values of states outside the current robust invariant set, therefore (\ref{second part of Lagrangian loss}) provides strong supervised signals for these states (the second term of (\ref{task policy loss}) becomes dominant). We summarize the overall procedure of DRAC in Algorithm \ref{DRAC}.

\begin{algorithm}[ht]
    \label{DRAC}
    \caption{Dually Robust Actor-Critic}
    \KwIn{network parameters $\theta$, $\phi$, $\beta$, $\xi$, $\omega_1$, $\omega_2$, $\psi$, $\zeta$, target network parameters $\bar{\psi}\leftarrow\psi$, $\bar{\omega}_1\leftarrow\omega_1$, $\bar{\omega}_2\leftarrow\omega_2$, temperature $\alpha$, learning rate $\eta$, target smoothing coefficient $\tau$, replay buffer $\mathcal{D}\leftarrow \varnothing$.}
    \For{each iteration}{
        \For{each system step}{
            Sample control input $u_t\sim \pi(x_t;\theta)$;
            
            Sample disturbance $a_t\sim \mu(\cdot\mid x_t;\beta)$ or $a_t= \mu_h(x_t;\xi)$ (random choice);

            Observe next state $x_{t+1}$, reward $r_t$, constraint value $h_t$;

            Store transition $\mathcal{D} \leftarrow \mathcal{D} \cup\left\{\left(x_t, u_t, a_t, r_t, h_t, x_{t+1}\right)\right\}$.
        }

        \For{each gradient step}{
            Sample a batch of data from $\mathcal{D}$;

            Update safety value function $\psi \leftarrow \psi-\eta \nabla_\psi L_{Q_h}(\psi)$;

            Update value functions $\omega_i \leftarrow \omega_i-\eta \nabla_{\omega_i} L_{Q}(\omega_i)$ for $i\in\left\{1,2\right\}$;

            Update task policy $\theta \leftarrow \theta-\eta \nabla_\theta L_{\pi}(\theta)$;

            Update safety policy $\phi \leftarrow \phi-\eta \nabla_\phi L_{\pi_h}(\phi)$;

            Update performance adversary $\beta \leftarrow \beta-\eta \nabla_\beta L_{\mu}(\beta)$;

            Update safety adversary $\xi \leftarrow \xi-\eta \nabla_\xi L_{\mu_h}(\xi)$;

            Update Lagrange multiplier $\zeta \leftarrow \zeta + \eta \nabla_\zeta L_{\lambda}(\zeta)$;

            Update temperature $\alpha \leftarrow \alpha-\eta \nabla_\alpha L(\alpha)$;

            Update target networks $\bar{\psi} \leftarrow \tau \psi+(1-\tau) \psi$, $\bar{\omega}_i \leftarrow \tau \omega_i+(1-\tau) \bar{\omega}_i$ for $i\in\left\{1,2\right\}$.
        }
    }
\end{algorithm}

\section{Experiments}

In this section, we evaluate our algorithm DRAC on safety-critical benchmark environments. We compare our algorithm with state-of-the-art safe RL and robust RL algorithms.

\subsection{Environments}

All algorithms are tested on three environments: CartPole, RacingCar and Walker2D.

\textbf{CartPole} is a classic control task based on MuJoCo \cite{todorov2012mujoco}, as illustrated in Fig. \ref{CartPole pic}. The goal is to push the cart to a target position as fast as possible. The state of the system includes cart position $x$, cart velocity $v$, pole angle $\theta$ and pole angular velocity $\omega$. The safety constraint is imposed on the pole angle: $|\theta|\leq0.2$. Both control inputs and external disturbances are level forces applied on the cart. The protagonist action space is $\mathcal{U}=[-1,1]$ and the adversary action space is $\mathcal{A}=[-0.5,0.5]$.

\textbf{RacingCar} is a safe RL benchmark based on PyBullet \cite{coumans2021}, as shown in Fig. \ref{RacingCar pic}. The four-wheeled car needs to track the edge of the blue region accurately and quickly. The state space $\mathcal{X} \subset \mathbb{R}^7$. The safety constraint is staying inside the region between the two yellow boundaries. The protagonist action space is $\mathcal{U}=[-1,1]^2$ and the adversary action space is $\mathcal{A}=[-0.25,0.25]^2$.

\textbf{Walker2D} is a safe RL benchmark based on MuJoCo \cite{todorov2012mujoco}, as shown in Fig. \ref{Walker2D pic}. The agent is a two-dimensional two-legged figure with a state space $\mathcal{X} \subset \mathbb{R}^{17}$. The goal is to move as far as possible with minimal control efforts. The safety constraint is imposed on the angle $\theta$ and height $z$ of the torso: $|\theta|\leq0.25, 1.0\leq z\leq 1.8$. The control inputs $u\in \mathbb{R}^6$ are torques applied on the hinge joints. The external disturbances $a \in \mathbb{R}^6$ are forces applied on the torso and both feet.

During training, we do not terminate the episodes when safety constraints are violated, since the early termination trick may confuse the reward-seeking and safety-preserving capability of safe RL algorithms, leading to vague results. Safe RL algorithms should figure out ways of maintaining safety purely based on the signals of constraint function $h(x)$. For algorithms without safety considerations, we equip them with the reward shaping trick (adding a bonus to the reward when no constraint is violated) for fair comparison.

\subsection{Baselines}

We compare DRAC to the following baselines.

\textbf{Soft Actor-Critic with Reward Shaping} (SAC-Rew, \cite{haarnoja2018soft}): The standard SAC algorithm with additional bonuses added to the original rewards for encouragement of constraint satisfaction.

\textbf{Robust Soft Actor-Critic with Reward Shaping} (RSAC-Rew, \cite{pinto2017robust}): The SAC version of robust adversarial reinforcement learning (RARL), a state-of-the-art robust RL algorithm, with additional bonuses added to the original rewards for encouragement of constraint satisfaction.

\textbf{Soft Actor-Critic with Lagrange Multiplier} (SAC-Lag, \cite{ha2020learning}): The combination of SAC and the method of Lagrange multipliers, a state-of-the-art safe RL algorithm.

\textbf{Reachable Actor-Critic} (RAC, \cite{yu2022reachability}): The combination of SAC and the reachability constraint, which makes use of Hamilton-Jacobi reachability analysis (under the assumption of no disturbances), a state-of-the-art safe RL algorithm.

\textbf{Soft Actor-Critic with Robust Invariant Set} (SAC-RIS, ours): A weaker implementation of our proposed DRAC, which only focuses on robustness against safety attacks. The performance adversary in DRAC is removed.

\begin{figure}
    \centering
    \subfloat[CartPole]{
        \label{CartPole pic}
        \includegraphics[width=1.05in]{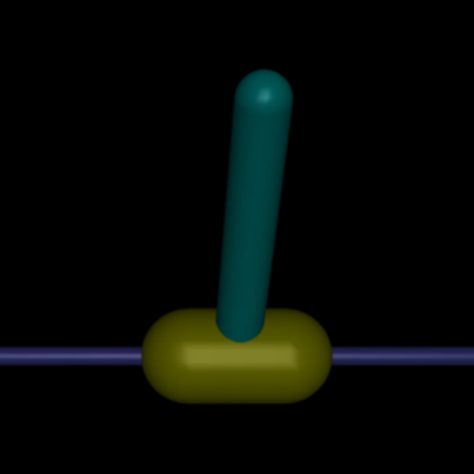}
    }
    \subfloat[RacingCar]{
        \label{RacingCar pic}
        \includegraphics[width=1.05in]{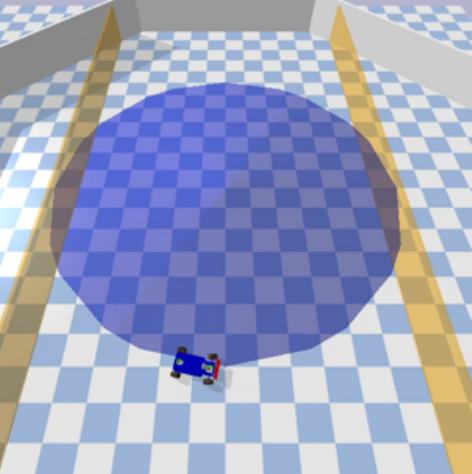}
    }
    \subfloat[Walker2D]{
        \label{Walker2D pic}
        \includegraphics[width=1.05in]{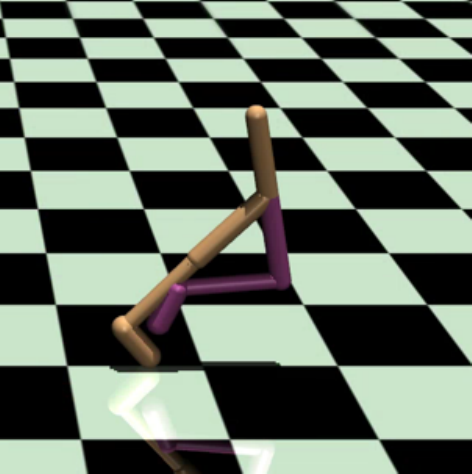}
    }
    \caption{Snapshots of three safety-critical environments.}
\end{figure}

\subsection{Results}

We adopt two evaluation metrics for each algorithm in the training process. (1) \textbf{episode return} indicates the performance of the learned agents. (2) \textbf{episode constraint violation} identifies the safety-preserving capability of the learned agents.

All algorithms are tested under three scenarios. (1) \textbf{no adversary}: There are no external disturbances in the environments. (2) \textbf{safety adversary}: The task policies are evaluated under the attacks of the learned safety adversary from our algorithm DRAC. This scenario compares the safety-preserving robustness of different algorithms. (3) \textbf{performance adversary}: The task policies are evaluated under the attacks of the learned performance adversary from our algorithm DRAC. This scenario compares the reward-seeking robustness of different algorithms.

The learning curves for three environments are shown in Fig. \ref{learning curves}. Since safety is the top priority, it is meaningless to attain high rewards when safety constraints are violated. SAC-RIS and DRAC are the only two algorithms that can maintain persistent safety under all scenarios. SAC-Rew and RSAC-Rew violate the safety constraints heavily, even in the case of no adversary. This is due to the lack of safety-preserving design in their mechanisms. Although the reward shaping trick encourages constraint satisfaction, it is far from enough to achieve persistent safety. SAC-Lag and RAC can maintain safety when no disturbances exist, but their safety-preserving capabilities are compromised under safety or performance attacks. Despite the fact that RAC-RIS is not trained with performance adversary, it exhibits excellent safety-preserving capability under the unseen performance attacks. This further justifies the concept of robust invariant sets, inside which the system can maintain safety under any form of safety adversaries. Compared to SAC-RIS, DRAC achieves considerably higher performance, which justifies the effectiveness of considering performance adversary. The policies learned by DRAC are robust to both performance and safety attacks. Furthermore, they attain higher rewards than SAC-Lag and RAC even in the absence of adversary.

\begin{figure*}[htbp]
    \centering
    \includegraphics[width=6.3in]{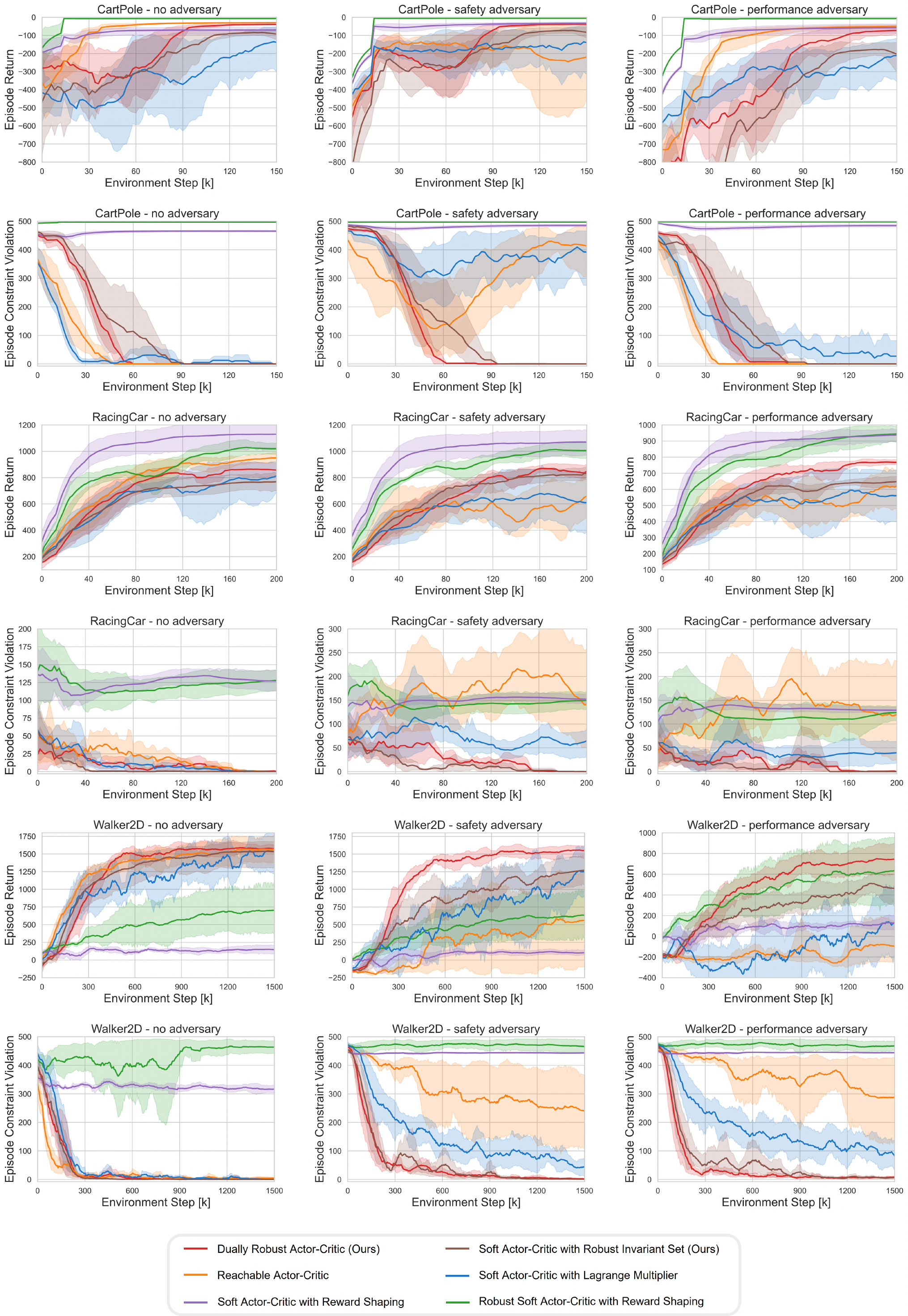}
    \caption{Training curves on three environments. For each environment, the first row corresponds to episode return and the second row corresponds to episode constraint violation. The solid lines correspond to the mean and the shaded regions correspond to 95\% confidence interval over five seeds.}
    \label{learning curves}
\end{figure*}

\section{Conclusion}

In this paper, we propose a systematic framework to unify safe RL and robust RL, including the problem formulation, iteration scheme, convergence analysis and practical algorithm design. The unification is built upon constrained two-player zero-sum Markov games, in which a twofold objective is designed. We propose a dual policy iteration scheme that jointly optimizes task policy and safety policy. Safety value function is utilized to characterize the riskiness of states under disturbances. The safety policy seeks the highest safety value for each state, identifying the robust invariant set, which serves as constraint for the task policy. We prove that the proposed scheme converges to the optimal task policy and the optimal safety policy. Furthermore, we propose dually robust actor-critic (DRAC), a deep RL algorithm that is robust to both safety and performance attacks. Experimental results demonstrate the effectiveness of our algorithm.

\bibliographystyle{IEEEtran}
\bibliography{reference}

\begin{IEEEbiography}[{\includegraphics[width=1in,height=1.25in,clip,keepaspectratio]{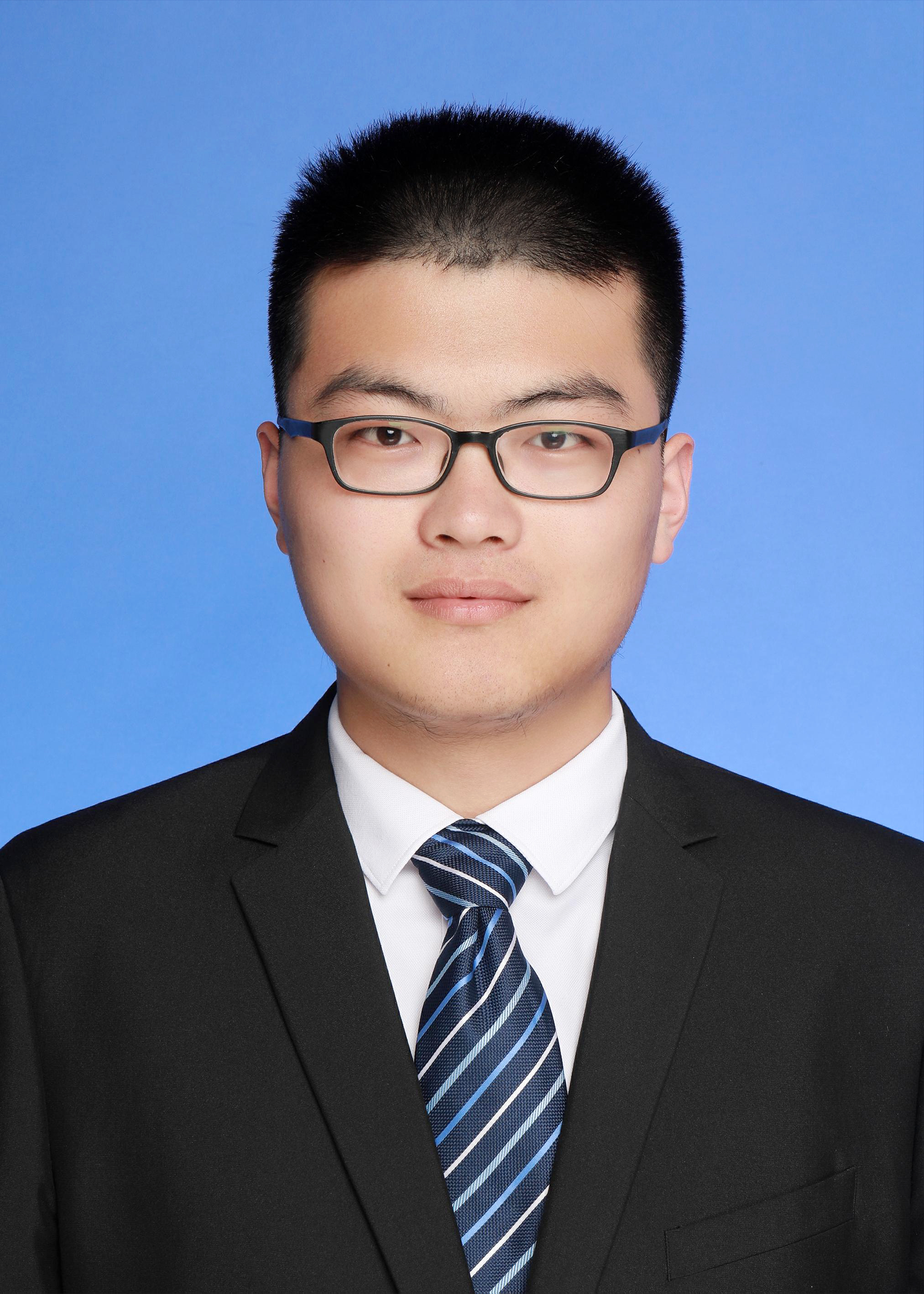}}]{Zeyang Li}received the B.S. degree in mechanical engineering in 2021, from the School of Mechanical Engineering, Shanghai Jiao Tong University, Shanghai, China. He is currently working toward the M.S. degree in mechanical engineering with the Department of Mechanical Engineering, Tsinghua University, Beijing, China. His research interests include reinforcement learning and optimal control.
\end{IEEEbiography}

\begin{IEEEbiography}[{\includegraphics[width=1in,height=1.25in,clip,keepaspectratio]{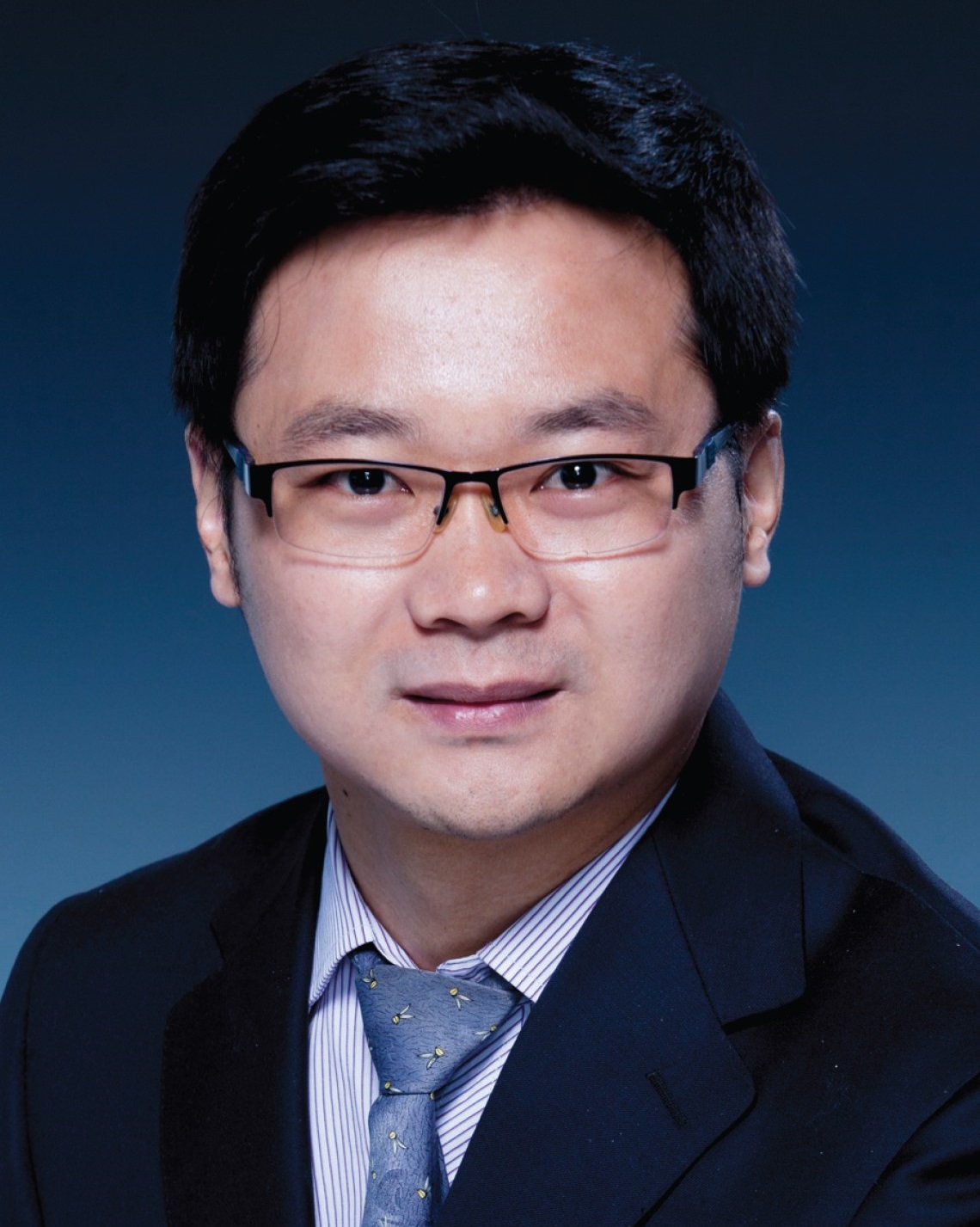}}]{Chuxiong Hu}(Senior Member, IEEE) received his B.S. and Ph.D. degrees in Mechatronic Control Engineering from Zhejiang University, Hangzhou, China, in 2005 and 2010, respectively. He is currently an Associate Professor (tenured) at Department of Mechanical Engineering, Tsinghua University, Beijing, China. From 2007 to 2008, he was a Visiting Scholar in mechanical engineering with Purdue University, West Lafayette, USA. In 2018, he was a Visiting Scholar in mechanical engineering with University of California, Berkeley, CA, USA. His research interests include precision motion control, high-performance multiaxis contouring control, precision mechatronic systems, intelligent learning, adaptive robust control, neural networks, iterative learning control, and robot. Prof. Hu was the recipient of the Best Student Paper Finalist at the 2011 American Control Conference, the 2012 Best Mechatronics Paper Award from the ASME Dynamic Systems and Control Division, the 2013 National 100 Excellent Doctoral Dissertations Nomination Award of China, the 2016 Best Paper in Automation Award, the 2018 Best Paper in AI Award from the IEEE International Conference on Information and Automation, and 2022 Best Paper in Theory from the IEEE/ASME International Conference on Mechatronic, Embedded Systems and Applications. He is now an Associate Editor for the IEEE Transactions on Industrial Informatics and a Technical Editor for the IEEE/ASME Transactions on Mechatronics.
\end{IEEEbiography}

\vspace{-5mm}

\begin{IEEEbiography}[{\includegraphics[width=1in,height=1.25in,clip,keepaspectratio]{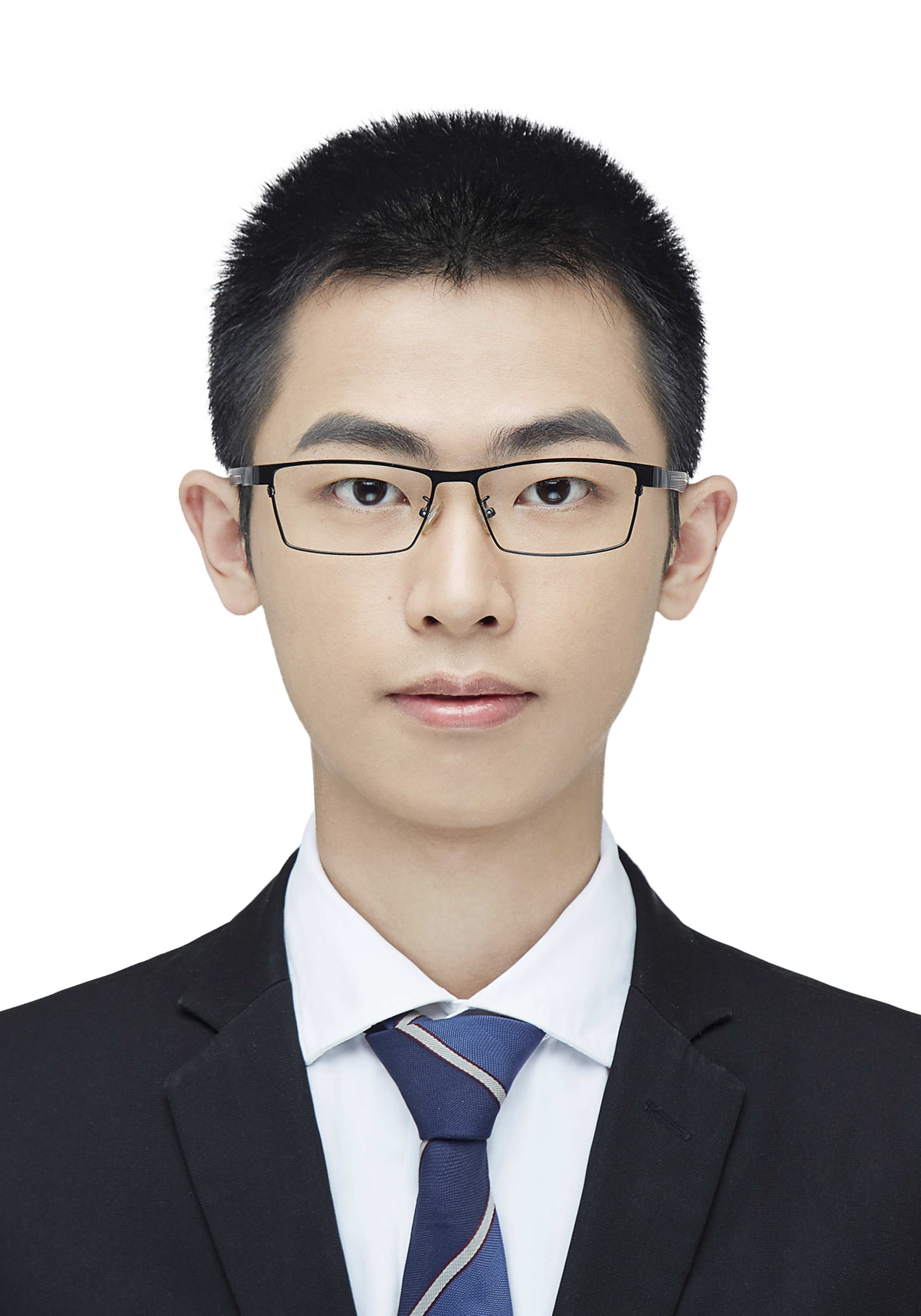}}]{Yunan Wang}(Graduate Student Member, IEEE) received the B.S. degree in mechanical engineering, in 2022, from the Department of Mechanical Engineering, Tsinghua University, Beijing, China. He is currently working toward the Ph.D. degree in mechanical engineering. His research interests include optimal control, trajectory planning, toolpath planning, and precision motion control. He was the recipient of the Best Conference Paper Finalist at the 2022 International Conference on Advanced Robotics and Mechatronics, and 2021 Top Grade Scholarship for Undergraduate Students of Tsinghua University.
\end{IEEEbiography}

\vspace{-5mm}

\begin{IEEEbiography}[{\includegraphics[width=1in,height=1.25in,clip,keepaspectratio]{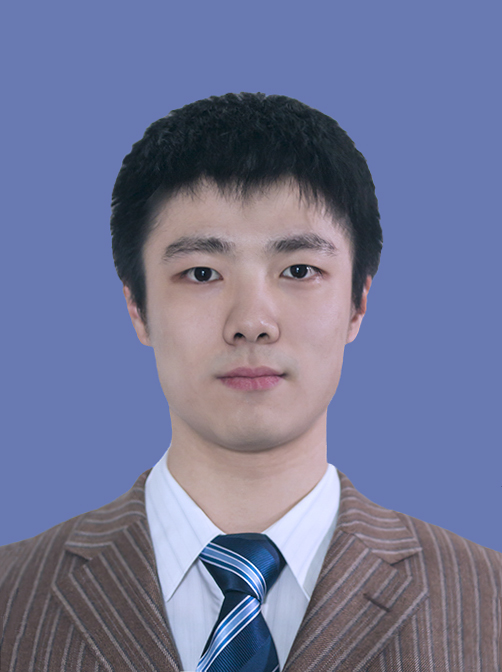}}]{Yujie Yang} received his B.S. degree in automotive engineering from the School of Vehicle and Mobility, Tsinghua University, Beijing, China, in 2021. He is currently pursuing his Ph.D. degree in the School of Vehicle and Mobility, Tsinghua University, Beijing, China. His research interests include decision and control of autonomous vehicles and safe reinforcement learning.
\end{IEEEbiography}

\vspace{-5mm}

\begin{IEEEbiography}[{\includegraphics[width=1in,height=1.25in,clip,keepaspectratio]{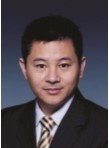}}]{Shengbo Eben Li}(Senior Member, IEEE) received his M.S. and Ph.D. degrees from Tsinghua University in 2006 and 2009. Before joining Tsinghua University, he has worked at Stanford University, University of Michigan, and UC Berkeley. His active research interests include intelligent vehicles and driver assistance, deep reinforcement learning, optimal control and estimation, etc. He is the author of over 130 peer-reviewed journal/conference papers, and the co-inventor of over 30 patents. He is the recipient of best (student) paper awards of IEEE ITSC, ICCAS, IEEE ICUS, CCCC, etc. His important awards include National Award for Technological Invention of China (2013), Excellent Young Scholar of NSF China (2016), Young Professor of ChangJiang Scholar Program (2016), National Award for Progress in Sci \& Tech of China (2018), Distinguished Young Scholar of Beijing NSF (2018), Youth Sci \& Tech Innovation Leader from MOST (2020), etc. He also serves as Board of Governor of IEEE ITS Society, Senior AE of IEEE OJ ITS, and AEs of IEEE ITSM, IEEE Trans ITS, Automotive Innovation, etc.
\end{IEEEbiography}

\end{document}